\documentclass[11pt]{article}
\usepackage{multirow}
\usepackage{color}
\usepackage{amsmath, amsthm, amssymb, amsfonts}
\usepackage{fullpage}
\usepackage[margin=0.9in]{geometry}
\usepackage{xspace}
\usepackage{comment}
\usepackage{dsfont}

\usepackage{algorithm}
\usepackage{algorithmic}

\RequirePackage{times}
\RequirePackage{natbib}

\usepackage{multirow}
\usepackage{color}
\usepackage{amsmath, amsthm, amssymb, amsfonts}
\usepackage{xspace}
\usepackage{comment}
\usepackage{xcolor}
\usepackage[pdftex,colorlinks,linkcolor=blue,citecolor=blue,filecolor=blue,urlcolor=blue]{hyperref}

\newcommand{\State}{\STATE}
\newcommand{\For}{\FOR}
\newcommand{\EndFor}{\ENDFOR}
\newcommand{\Return}{\textbf{Return}}

\ifx\theorem\undefined
\newtheorem{theorem}{Theorem}
\fi
\ifx\example\undefined
\newtheorem{example}{Example}
\fi
\ifx\property\undefined
\newtheorem{property}{Property}[theorem]
\fi
\ifx\lemma\undefined
\newtheorem{lemma}[theorem]{Lemma}
\fi
\ifx\proposition\undefined
\newtheorem{proposition}{Proposition}
\fi
\ifx\remark\undefined
\newtheorem{remark}{Remark}
\fi
\ifx\corollary\undefined
\newtheorem{corollary}[theorem]{Corollary}
\fi
\ifx\definition\undefined
\newtheorem{definition}{Definition}
\fi
\ifx\conjecture\undefined
\newtheorem{conjecture}[theorem]{Conjecture}
\fi
\ifx\fact\undefined
\newtheorem{fact}{Fact}
\fi
\ifx\claim\undefined
\newtheorem{claim}[theorem]{Claim}
\fi
\ifx\assump\undefined
\newtheorem{assump}{Assumption}
\fi
\ifx\cond\undefined
\newtheorem{cond}[theorem]{Condition}
\fi

\newcommand{\RR}{\mathbb{R}}
\newcommand{\cS}{\mathcal{S}}
\newcommand{\cF}{\mathcal{F}}

\newcommand{\cT}{\mathcal{T}}
\newcommand{\PP}{\Pr}
\newcommand{\EE}{\mathbb{E}}

\newcommand{\cE}{\mathcal{E}}

\DeclareMathOperator{\poly}{poly}

\newcommand{\cA}{\mathcal{A}}

\newcommand{\cG}{\mathcal{G}}
\newcommand{\cK}{\mathcal{K}}

\newcommand{\wt}[1]{\widetilde{#1}}
\newcommand{\wh}[1]{\widehat{#1}}
\newcommand{\wb}[1]{\overline{#1}}

\newcommand{\codecmt}[1]{\textcolor{blue!10!black!80!green!70!}{$\triangleright$~#1}}

\newcommand{\cM}{\mathcal{M}}

\DeclareMathOperator{\var}{Var}

\newcommand{\colnote}[3]{}

\title{Sample-Optimal Parametric Q-Learning Using Linearly Additive Features}

\author{
	Lin F. Yang%
	\thanks{Email: \texttt{lin.yang@princeton.edu }
	}
	\\
	Princeton University
	\and
	Mengdi Wang%
	\thanks{Email: \texttt{mengdiw@princeton.edu}
}
\\
Princeton University	
}

\begin{document}

\maketitle

\begin{abstract}
Consider a Markov decision process (MDP) that admits a set of state-action features, which can linearly express the process's probabilistic transition model. We propose a parametric Q-learning algorithm that finds an approximate-optimal policy using a sample size proportional to the feature dimension $K$ and invariant with respect to the size of the state space. To further improve its sample efficiency, we exploit the monotonicity property and intrinsic noise structure of the Bellman operator, provided the existence of anchor state-actions that imply implicit non-negativity in the feature space. We augment the algorithm using techniques of variance reduction, monotonicity preservation, and confidence bounds. It is proved to find a policy which is $\epsilon$-optimal from any initial state with high probability using $\wt{O}(K/\epsilon^2(1-\gamma)^3)$ sample transitions for arbitrarily large-scale MDP with a discount factor $\gamma\in(0,1)$. A matching information-theoretical lower bound is proved, confirming the sample optimality of the proposed method with respect to all parameters (up to polylog factors).
\end{abstract}

\section{Introduction}

Markov decision problems (MDP) are known to suffer from the curse of dimensionality. A basic theoretical question is: Suppose that one can query sample transitions from any state of the system using any action, how many samples are needed for learning a good policy? In the tabular setting where the MDP has $S$ states and $A$ actions, the necessary and sufficient sample size for finding an approximate-optimal policy is $ \wt{\Theta}(\frac{SA}{(1-\gamma)^3})$
\footnote{$\wt{f(\cdot)}$ ignores $\poly\log f(\cdot)$ factors.}
 where $\gamma\in(0,1)$ is a discount factor \cite{azar2013minimax, sidford2018near}. However, this theoretical-sharp result does not generalize to practical problems where $S,A$ can be arbitrarily large or infinite.

Let us consider MDP with structural knowledges. Suppose that each state-action pair $(s,a)$ admits a feature vector $\phi(s,a)\in\RR^K$ that can express the transition dynamics conditioning on $(s,a)$. 
In practice, the abstract state variable $s$ can be a sequence of historical records or a raw-pixel image, containing much information that is not related to the decision process. 
More general settings of MDP with structural knowledges have been considered in \citet{azizzadenesheli2016reinforcement,jiang2017contextual} and references therein.

In this paper, we focus on an important and very basic class of structured MDP, where the features can represent transition distributions $P(\cdot\mid \cdot)$ through an unknown {\it linear additive model}. The feature-based linear transition model is related to the commonly used linear Q-function model. We show that they are essentially equivalent when there is zero Bellman error (a notion introduced in \citet{ms08}). A similar argument has been made in \citet{parr2008analysis}. It also contains as a special case the soft state aggregation model \cite{singh1995reinforcement, duan2018state}.
In this setting, we will study the theoretic sample complexity for learning a good policy by querying state-transition samples. We also aim to develop efficient policy learning algorithms with provable sample efficiency. We study the following two questions: 


\vspace{0.5pt}
{\bf Q1:} How many observations of state-action-state transitions are {\it necessary} for finding an $\epsilon$-optimal policy? 

\vspace{0.5pt}
{\bf Q2:} How many samples are {\it sufficient} for finding an $\epsilon$-optimal policy with high probability  and {\it how} to find it? 

\vspace{0.5pt}
To answer {\bf Q1}, an information-theoretic lower bound is provided (Theorem~\ref{theorem-lower}), suggesting that, regardless of the learning algorithm, the necessary sample size for finding a good policy with high probability is $
	\wt{\Omega}\big(\frac{K}{(1-\gamma)^3\cdot\epsilon^2}\big)
	$ where $K$ is the dimension of feature space.
	
To answer {\bf Q2}, we develop Q-learning-like algorithms that take as input state-transition samples and output a parameterized policy. A basic parametric Q-learning algorithm performs approximate value-iteration estimates on a few points of the Q function, so that actual updates happen on the parameters. This idea originates from the phased Q-learning \cite{kearns1999finite} and the fitted value iteration \cite{ms08, asm08,asm08a}. Our algorithm is simpler and does not require function fitting. Convergence and approximation error analysis is provided even when the MDP cannot be fully expressed using the features. Despite its simplicity, the basic algorithm has complexity $\wt{O}\big(\frac{K}{(1-\gamma)^7 \cdot\epsilon^2}\big)$, which is not sample-optimal. 

Furthermore, we develop an accelerated version of parametric Q-learning that involves taking mini-batches, computing confidence bounds, and using monotonicity-preserving and variance reduction techniques. It uses some ideas from fast solvers of tabular MDP \cite{sidford2018variance, sidford2018near}. 
To fully exploit the monotonicity property of the Bellman operator in the algorithm, we need an additional ``anchor" assumption, i.e., there exists a (small) set of state-actions that can represent the remaining ones using convex combinations. The ``anchors" can be viewed as vertices of the state-action space, and implies an intrinsic nonnegativity in the feature space which is needed for monotonic policy improvement. We show that the algorithm takes just enough  samples per update to keep the value/policy iterates within a sequence of narrow confidence regions that monotonically improve to the near-optimal solutions.  It finds an $\epsilon$-optimal policy (regardless of the initial state) with probability at least $1-\delta$ using  
	$$
	\wt{\Theta}\bigg(\frac{K}{(1-\gamma)^3 \cdot\epsilon^2}\cdot\log\frac1\delta\bigg)
	$$
	samples. It matches the information-theoretic lower bound up to $\log(\cdot)$ factors, thus the algorithm is nearly sample-optimal. If $\gamma=0.99$, this algorithm is $(1-\gamma)^{-4} = 10^8$ times faster than the basic algorithm.


Our model, algorithms and analyses relate to previous literatures on the sample complexity of tabular MDP, reinforcement learning with function approximation, linear models and etc. A detailed account for the related literatures is given in Section \ref{related-work}. All technical proofs are given in the appendix. To our best knowledge, this work provides the first sample-optimal algorithm and sharp complexity analysis (up to polylog factors) for MDP with linear models.

%
%

\section{Markov Decision Process, Features, Linear Models}
\label{sec:probdef}

In this section we introduce the basics of Markov decision process and the feature-based linear transition model. 

\subsection{Preliminaries}

In a \emph{discounted Markov decision process} (DMDP or MDP for short),
there is a finite set of \emph{states} $\cS$, a finite set of \emph{actions} $\cA$. Let $S=|\cS|$ and $A=|\cA|$.
At any state $s\in \cS$, an agent is allowed to play an action $a\in \cA$.
She receives an immediate reward $r(s, a)\in [0, 1]$ after playing $a$ at $s$,
and then the process will transition to the next state $s'\in \cS$ with probability $P(s'|s,a)$, where $P$ is the collection of \emph{transition distributions}.
The full instance of MDP can be described by the tuple $M=(\cS,\cA,P,r,\gamma).$
The agent would like to find a \emph{policy} $\pi:\cS\rightarrow \cA$ that maximizes the long-term expected reward starting from every state $s$, i.e.,
\[
v^{\pi}(s) := \EE\bigg[\sum_{t=0}^\infty \gamma^tr(s^t, \pi(s^t))|s^0 = s\bigg]
\]
where $\gamma \in(0,1)$ is a discount factor.
We call $v^{\pi}\in \RR^{\cS}$ the \emph{value function} of policy $\pi$.
A policy $\pi^*$ is said to be \emph{optimal} if it attains the maximal possible value at every state.
In fact, it is known (see e.g. \citet{puterman2014markov}) that there is a \emph{unique} optimal value function  $v^*$ such that
\[
\forall s\in \cS: \quad v^*(s) = \max_{\pi}  v^{\pi}(s) = v^{\pi^*}(s).
\]
A policy $\pi$ is said to be \emph{$\epsilon$-optimal} if it achieves near-optimal cumulative reward from {\it any} initial state, i.e.,
$$v^{\pi}(s) \geq v^*(s) -\epsilon,\qquad\forall\ s\in\cS, $$
or equivalently $\|v^{\pi}-v^*\|_{\infty}\le \epsilon$ for short.
We denote the Bellman operator $\cT:\RR^{\cS}\to \RR^{\cS}$ as
\[
\forall s\in \cS: \quad [\cT v](s) = \max_{a\in \cA}[r(s,a) + \gamma P(\cdot|s,a)^\top v].
\]
A vector $v^*$ is the optimal value of the DMDP if and only if it satisfies the {\it Bellman equation} $v= \cT v$.

The Q-function of a policy $\pi$ is defined as $Q^{\pi} (s,a) = r(s,a) + \gamma \sum_{s'} P(s'|s,a) v^{\pi}(s')$, and the optimal Q-function is denoted by $Q^*=Q^{\pi^*}.$
We overload the notation $\cT$ to also denote the Bellman operator in the space of Q-functions, i.e., $\cT:\RR^{\cS\times \cA}\to \RR^{\cS\times \cA}$ such that
$$\mathcal{T}Q (s,a) = r(s,a) + \gamma P(\cdot|s,a)^\top\max_{a'}Q(\cdot,a').$$
A vector $Q^*\in\RR^{\cS\times \cA}$ is the optimal Q-function if and only if it satisfies the Bellman equation $Q=\cT Q.$

We use $O$,$\Omega$, $\Theta$ to denote leading orders, and we use $\tilde O$,$\tilde \Omega$, $\tilde\Theta$ to omit polylog factors. We use $\lesssim$ to denote ``approximately less than" by ignoring non-leading order terms, constant and polylog factors. 

\subsection{Feature-based Linear Transition Model}

We study Markov decision processes with structural knowledges. Suppose that the learning agent is given a set of $K$ feature functions $\phi_1, \phi_2, \ldots, \phi_K: \cS\times \cA\rightarrow \RR$.
The feature $\phi$ maps the raw state and action $(s,a)$ into the $K$-dimensional vector  
$$\phi(s,a) = [\phi_1(s,a), \phi_2(s,a), \ldots, \phi_K(s,a)]\in \RR^{K}.$$
Suppose the feature vector $\phi(s,a)$ is sufficient to express the future dynamics of the process conditioning on the current raw state and action. In particular, we focus on a basic linear model given below. 

\begin{definition}[Feature-based Linear Transition Model]
	\label{assum:context}
	Consider a DMDP instance $M=(\cS,\cA,P,r,\gamma)$ and a feature map $\phi: {\cS\times \cA}\to\RR^K.$
	We say that $M$ admits a {linear feature representation} $\phi$ if for every $s,a,s'$, 
	\[
	P(s' |s,a) = \sum_{k\in [K]} \phi_k(s,a) \psi_k(s') .
	\]
	for some functions $\psi_1,\ldots,\psi_{K}:\cS\to\RR.$ We denote the set of all such MDP instances as  $\mathcal{M}^{trans}(\cS,\cA,\gamma, \phi)$.
	We denote $\mathcal{M}^{trans}_K(\cS,\cA,\gamma)$ the set of all DMDP instances that admits a $K$-dimensional feature representation. 
\end{definition}

\begin{remark}[Independence of rewards] \rm
The feature representations $\phi(s,a)$ in Definition \ref{assum:context} capture the transition dynamics of the Markov process under different actions. It is a form of structural knowledge about the environment. It has nothing to do with the rewards $r(s,a)$.
\end{remark}

\begin{remark}[Combining state features and action features]\rm
In many settings one may be given a state-only feature map $\phi_1$ and an action-only feature map $\phi_2$. In this case, one can construct the joint state-action feature by $\phi(s,a) = \phi_1(s)\phi_2(a)$. As long as the MDP admits a linear transition model in both $\phi_1,\phi_2$, it also admits a linear representation in the product feature $\phi= \phi_1\times\phi_2$.
\end{remark}

\begin{remark}[Relation to soft state aggregation]\rm
The feature-based linear transition model (Definition \ref{assum:context}) contains a worth-noting special case. When each $\phi(s,a)\in\RR^K$ and $\psi_k\in\RR^{S}$ is a probability density function, the linear transition model reduces to a soft state aggregation model \cite{singh1995reinforcement, duan2018state}. In the soft state aggregation model, each state can be represented by a mixture of latent meta-states, through aggregation and disaggregation distributions. There would be $K$ meta-states, which can be viewed as the leading ``modes" of the process. 
In contrast, our feature-based transition model is much more general. Our feature map $\phi$ can be anything as long as it is representative of the transition distributions. It captures information about not only the states but also the actions. 
\end{remark}

\subsection{Relation to Linear Q-function Model}

Linear models are commonly used for approximating value functions or Q-functions using given features (sometimes referred to as basis functions). 
The proposed linear transition model is closely related to the linear Q-function model, where $Q^{\pi}$'s are assumed to admit a linear representation.

First it is easy to see that if the MDP admits a linear transition model using $\phi$, the Q-functions admit a linear model.
\begin{proposition}
\label{prop:relatelm}
Let $M \in \mathcal{M}^{trans}(\cS,\cA,\gamma, \phi)$. Then $Q^{\pi} \in \hbox{Span}(r,\phi)$ for all $\pi$.
\end{proposition}

Next we show that the two models are essentially ``equivalent" in terms of expressibility. 
A similar conclusion was made by \citet{parr2008analysis}, which showed that a particular solution obtained using the linear value model is equivalent to a solution obtained using a linear transition model. 
Recall a notion of Bellman error that was firstly introduced in \cite{ms08}. \begin{definition}[Bellman Error]
Let $\cF\subset \RR^{\cS\times\cA}$ be a class of Q functions. 
Given the Bellman operator $\cT$, the Bellman error of $\cF$ is 
$
d(\cT\cF, \cF) = \sup_{g\in \cF}\inf_{f\in \cF} \|f - \cT g\|.
$
\end{definition}

We show that the linear transition model is equivalent to the linear Q-function model with zero Bellman error.

\begin{proposition}[Equivalence to Zero Bellman Error]\label{prop-equiv}
Let $M=(\cS,\cA,P,r,\gamma)$ be an MDP instance with the Bellman operator $\cT$.
Let $\phi:\RR^{\cS\times \cA}$ be a feature map, and let $\cF =  \hbox{Span}(r, \phi)$. If $r\in \cF$, then 
$$d(\cT\cF, \cF) =0\ \ \  \hbox{if \& only if }\ \ 
M \in \mathcal{M}^{trans}(S,A,\gamma, \phi).$$
\end{proposition}

Suppose $Q^{\pi}\in\hbox{Span}(r,\phi)$ for all $\pi$'s. However, value-iteration-based method would still fail if the Bellman operator $\cT$ does not preserve the $(r,\phi)$ representation. 
In contrast, if the Q-functions admit linear representations using $\phi$ but the transition kernel $P$ does not, the Bellman error can be arbitrarily large. The Bellman error may be large even after projection or function fitting - a common source of unstable and oscillating behaviors in approximate dynamic programming \cite{tsitsiklis1996feature, ms08}. 



\section{Information-Theoretic Sample Complexity}

Let us study the feature-based MDP model (Definition \ref{assum:context}). It comes with the structural knowledge that each state-action pair $(s,a)$ can be represented by the feature vector $\phi(s,a)\in\RR^{K}$. However, this model can {\it not} be parameterized by a small number of parameters. The full transition model with known feature map $\phi$ can not be specified unless all the unknown parameters $\psi_k(s'),$ for $s'\in S,k\in[K]$ are given. Its model size is 
$S\times K,$ which can be arbitrarily large for arbitrarily large $S$. 

Given the state-action features, we aim to learn a near-optimal parametrized policy using a small number of samples, which hopefully depends on $K$ but not $S$. Suppose that we are given a \emph{generative model} \cite{kakade2003sample} where the agent is able to query transition samples and reward from any state-action pair $(s,a)\in \cS\times \cA$. Such a generative model is commonly available in simulation systems. 
To this end, we ask how many samples are necessary to obtain an approximate-optimal policy?
Our first theorem provides a firm answer.

\begin{theorem}[Sample Complexity Lower Bound\label{theorem-lower}]
	Let $M=(\cS,\cA,P,r,\gamma)$ be an instance of DMDP, and 
	let $\cA$ be any algorithm that queries sample transitions of $M$ and outputs a policy. 
	Let $\pi^{\cA, M,N}$ be the output of $\cA$ using $N$ samples. 
	Then
	\begin{align*}
	\inf_{\cA}&\sup_{M\in \cM^{trans}_K(\cS,\cA, \gamma)} \mathbb{P}\bigg( \sup_{s\in\cS} (v^*(s)-v^{\pi^{\cA, M,N}}(s) )\ge \epsilon  \bigg) \\&
	\geq 1/3,\quad
	\text{if} \quad N = 
	O\bigg(\frac{K}{(1-\gamma)^3\cdot\epsilon^2\cdot\log\epsilon^{-1}}\bigg),
	\end{align*}
	provided $\epsilon\le \epsilon_0$ for some $\epsilon_0 \ge 0$.
\end{theorem}
Theorem \ref{theorem-lower} suggests that, in order to solve the feature-based MDP to precision level $\epsilon$ with probability at least $2/3$, any algorithm needs at least $\wt\Omega\bigg(\frac{K}{(1-\gamma)^3\cdot\epsilon^2}\bigg)$ sample transitions in the worst case.

In the tabular setting without any feature, the sample complexity lower bound is known to be $\Omega(\frac{SA}{(1-\gamma)^3\epsilon^2})$ \cite{azar2013minimax}. Our lower bound can be proved by constructing a reduction from the feature-based model to a smaller-size tabular MDP with $K$ state-action pairs. In the reduction, the features are constructed as indicator functions corresponding to a $K$-partition of the state space. We postpone the proof to the appendix.



\section{A Basic Parametric Q-Learning Method}
We develop a Q-learning algorithm for MDP admitting feature representations provided with a generative model.

\subsection{Algorithm}

Recall that 
phased Q-Learning \cite{kearns1999finite} takes the form
${Q}(s,a) \leftarrow r(s,a) + \frac{\gamma}{m}\sum_{i=1}^m$ $\max_{a'}$ $Q(s'_i,a')$,
where $s'_i$'s are sample states generated from $P(\cdot\mid s,a).$
In the tabular setting, one needs to keep track of all the $Q(s,a)$ values.

Given the feature map $\phi:\cS\times\cA\to\RR^K$, we parameterize the Q-functions, value functions and policies using $w\in\RR^K$ by
\begin{align}
\label{eqn:value-policy}
Q_w(s,a) &:=  r(s,a) + \gamma \phi(s,a)^\top w,\\
V_{w}(s) &:= \max_{a\in \cA} Q_w(s,a),\\ 
\pi_{w}(s) &:= \arg\max_{a\in \cA} Q_w(s,a).
\end{align}
A scalable learning algorithm should keep track of only the parameters $w$, from which one can decode the high-dimensional value and policy functions according to (1-3).


Algorithm~\ref{alg:learn agm} gives a parametric phased Q-learning method. It queries state-action transitions and makes Q-learning-like updates on the parameter $w$. Each iteration picks a small set of state-action pairs  $\cK$, and performs approximate value iteration on $\cK$. The set $\cK$ can be picked almost arbitrarily. To obtain a convergence bound, we assume that the state-action pairs in $\cK$ cannot be too alike, i.e., the regularity condition \eqref{regularity} holds for some value $L>0$. 

\begin{assump}[\bf Representative States and Regularity of Features]
	\label{assum:reg}
	There exists a {representative state-action set} $\cK \subset \cS\times\cA$ with $|\cK|=K$ and a scalar $L>0$ such that 
	\begin{equation}\label{regularity}
	\|\phi(s,a)^T\Phi_{\cK}^{-1}\|_{1} \le L,\qquad \forall~(s,a)
	\end{equation}
	where $\Phi_{\cK}\in \RR^{K\times K}$ is the collection of row feature vectors $\phi(s,a)$ where $(s,a)\in\cK$ and $L \ge 1$.

\end{assump}


{\small
\begin{algorithm}[htb!]
	\caption{Phased Parametric Q-Learning (PPQ-Learning)\label{alg:learn agm}}
	\begin{algorithmic}[1]
		\State 
		\textbf{Input:} A DMDP $\cM=(\cS, \cA, P, r, \gamma)$ with a generative model
		\State \textbf{Input:} Integer $N > 0$
		\State 
		\State
		\textbf{Initialize:} $R\gets \Theta\big[\frac{\log N}{1-\gamma}\big]$, 
		$w\gets 0\in \RR^K$;
		\State
		\textbf{Repeat:}
		\For{$t=1,2,\ldots, R$}
		\State Pick a representative set $\cK\subset \cS\times \cA$ satisfying \eqref{regularity}.
		\State $Q\gets 0\in \RR^{K}$;
		\For{$(s,a)\in \cK$}
		\State Obtain $\frac{N}{KR}$ samples $\{s^{(j)}\}$ i.i.d. from $P(\cdot| s,a)$;	
		\State  $Q[(s,a)] \gets \frac{KR}{N}\sum_{j=1}^{N/{KR}}\Pi_{[0,(1-\gamma)^{-1}]}[V_{w}(s^{(j)})]$;
		\State$\qquad\qquad$ \codecmt{\small$\Pi_{[a,b]}$ projects a number onto $[a,b]$}
		\EndFor
		\State $w\gets \Phi_{\cK}^{-1} Q$;
		\EndFor
		\State 
		\textbf{Output:} $w\in \RR^{K}$
	\end{algorithmic}
\end{algorithm}
}

\subsection{Error Bound and Sample Complexity}

We show that the basic parametric Q-learning method enjoys the following error bound.


\begin{theorem}[Convergence of Algorithm \ref{alg:learn agm}]
	\label{thm:agm-result}
	Suppose Assumption \ref{assum:reg} holds. Suppose that the DMDP instance $M=(\cS,\cA,P,r,\gamma)$ has an approximate transition model $\wt{P}$ that admits a linear feature representation $\phi$ (Defn. \ref{assum:context}), such that for some $\xi\in[0,1]$, 
	$\|P(\cdot \mid s,a) -\wt{P}(\cdot \mid s,a) \|_{TV}\leq \xi$, $\forall~(s,a)$. 
	Let
	Algorithm~\ref{alg:learn agm} takes $N>0$ samples
	and outputs a parameter $w\in \RR^{K}$.
	Then, with probability at least $1-\delta$,  $\|v^{\pi_{w}} - v^*\|_{\infty}$
	\begin{align*}\le
	L\cdot \bigg(\sqrt{\frac{K}{N\cdot(1-\gamma)}}+\xi \bigg)\cdot \frac{ \poly\log(NK\delta^{-1})}{(1-\gamma)^3}. 
	\end{align*}
\end{theorem}


\begin{remark}[\bf Policy optimality guarantee]\rm Our bound applies to $v^{\pi_w}$, i.e., the actual performance of the policy $\pi_w$ in the real MDP. 
It is for the $\ell_\infty$ norm, i.e., the policy is $\epsilon$-optimal from every initial state. This is the strongest form of optimality guarantee for solving MDP.
\end{remark}

\begin{remark}[\bf Approximation error due to model misspecification]\rm
 When the feature-based transition model is inexact up to $\xi$ total variation, there is an approximation gap in the policy's performance 
$O\Big[L\cdot \xi \cdot \frac{\poly\log(NK\delta^{-1})}{(1-\gamma)^3} \Big].$
It suggests that, even if the observed feature values $\phi(s,a)$ cannot fully express the state and action, the Q-learning method can still find approximate-optimal policies. The level of degradation depends on the total-variation divergence between the true transition distribution and its closest feature-based transition model.
\end{remark}

\begin{remark}[\bf Sample complexity of Algorithm \ref{alg:learn agm}]\rm
When the MDP is fully realizable under the features, we have $\xi=0$. Then the number of samples needed for achieving $\epsilon$ policy error is 
$$\tilde O\bigg[\frac{K L^2}{(1-\gamma)^7\epsilon^2}\bigg].$$ 
It is independent of size of the original state space, but depends linearly on $K$. 
Its dependence on $\frac1{1-\gamma}$ matches the tabular phased Q-learning \cite{kearns1999finite} which has complexity $O(\frac{SA}{(1-\gamma)^7\epsilon^2})$ \cite{sidford2018near}. Despite the fact that the MDP model has $S\times K$ unknown parameters, the basic parametric Q-learning method can produce good policies even with small data.
However, there remains a gap between the current achievable sample complexity (Theorem \ref{thm:agm-result}) and the lower bound (Theorem \ref{theorem-lower}).
\end{remark}



\section{Sample-Optimal Parametric Q-Learning}

In this section we will accelerate the basic parametric Q-learning algorithm to maximize its sample efficiency. 
To do so, we need to modify the algorithm in nontrivial ways in order to take full advantage of the MDP's structure. 

%



\subsection{Anchor States and Monotonicity}

In order to use samples more efficiently, we need to leverage monotonicity of the Bellman operator (i.e., $\cT v_1\leq \cT v_2$ if $v_1\leq v_2$). However, when the $Q$ function is parameterized as a linear function in $w$, noisy updates on $w$ may easily break the pointwise monotonicity in the $Q$ space.
To remedy this issue, we will impose an additional assumption to ensure that monotonicity can be preserved implicitly. 


\begin{assump}[Anchor State-Action Pairs]
\label{assum:lgm}
There exists a set of anchor state-action pairs $\mathcal{K} $ such that for any $(s,a)\in \cS\times \cA$, its feature vector can be represented as a convex combination of the anchors $\{(s_k,a_k)\mid k\in\cK\}$:
\begin{align*}\exists \{\lambda_k\}:
\phi(s,a) = &\sum_{k\in \mathcal{K}} \lambda_k \phi(s_k,a_k),\quad \sum_{k\in \mathcal{K}}\lambda_k=1,\lambda_k\geq0.
\end{align*}
\end{assump}

The anchoring $(s_k,a_k)$'s can be viewed as ``vertices" of the state-action space. They imply that the transition kernel $P$ admits a nonnegative factorization, which can be seen by transforming $\phi$ linearly such that each anchor corresponds to a unit feature vector. This implicit non-negativity is a key to pointwisely monotonic policy/value updates. 

The notion of ``anchor" is a natural analog of the anchor word condition from topic modeling \cite{Ge} and nonnegative matrix factorization \cite{donoho2004does}. A similar notion of ``anchor state" has been studied in the context of soft state aggregation models to uniquely identify latent meta-states \cite{duan2018state}. 
Under the anchor assumption, without loss of generality, we will assume that $\phi$'s are nonnegative, each $\phi(s,a)$ is a vector of probabilities, and there are $K$ anchors with unit feature vectors. 



\subsection{Achieving The Optimal Sample Complexity}

We develop a sample-optimal algorithm which is implemented in Algorithm~\ref{alg:learn}. Let us explain the features that enable it to find more accurate policies. Some of the ideas are due to \cite{sidford2018variance, sidford2018near}, where they were used to develop fast solvers for the tabular MDP.

\noindent\textbf{Parametrization.} For the purpose of preserving monotonicity, Algorithm~\ref{alg:learn} employs a new parametric form. It uses a collection of parameters $\theta=\{w^{(i)}\}_{i=1}^{Z}$ instead of a single vector, with $Z= \wt{O}(\frac1{1-\gamma})$. 
The parameterized policy and value functions take the form 
\begin{align}
\label{eqn:vpidecoder1}
V_{\theta}(s)&:=\max_{h\in[Z]}
\max_{a\in \cA}\big(r(s,a) + \gamma\phi(s,a)^\top\cdot w^{(h)}
\big)\quad\text{and}\nonumber\\
\pi_{\theta}(s)&\in \arg
\max_{a\in \cA} \max_{h\in [Z]}\big(r(s,a) + \gamma\phi(s,a)^\top\cdot w^{(h)}\big).
\end{align}
Given $\theta$, one can compute $V_{\theta}(s),\pi_{\theta}(s)$ by solving an one-step optimization problem. If $a$ takes continuous values, it needs to solve a nonlinear optimization problem.

\noindent\textbf{Computing confidence bounds.} In Step 13 and Step 18, the algorithm computes confidence bounds $\epsilon^{(i,j)}$'s for the estimated values of $PV_{\theta}.$ 
These bounds tightly measure the distance from $V_{\theta}$ to the desired solution path, according to probaiblistic concentration arguments.
With these bounds, we can precisely shift our estimator downwards so that certain properties would hold (e.g. monotonicity to be explained later) while not incurring additional error.

\noindent\textbf{Monotonicity preservation.}
The algorithm guarantees that the following condition holds throughout:
\[
V_{\theta} \le \cT_{\pi_{\theta}}V_{\theta},\qquad \text{pointwise}.
\]
We call this property the \emph{monotonicity} property, which together with monotonicity of the Bellman operator guarantees that (by an induction proof)
\begin{align*}
V_{\theta}\le  \cT_{\pi_{\theta}}V_{\theta}\leq  \cT_{\pi_{\theta}}^2V_{\theta} \le\cdots \le \cT_{\pi_{\theta}}^{\infty}V_{\theta} = v^{\pi_{\theta}} &\le v^*, \\
& \text{pointwise}.
\end{align*}
Algorithm \ref{alg:learn} uses two algorithmic tricks to preserve the monotonicity property throughout the iterations. First, the parametric forms of $V_{\theta}$ and $\pi_{\theta}$ (eq.\eqref{eqn:vpidecoder1}) take the maximum across all previous parameters (indexed by $h=(i,j)$). It guarantees that $V_{\theta}$ is monotonically improving throughout the outer and inner iterations.
Second, the algorithm shifts all the estimated $V_{\theta}$ downwards by a term  corresponding to its confidence bound (last equation of Line 13 and Line 18 of Algorithm~\ref{assum:lgm}). 
As a result, the estimated expectation is always smaller than the true expected value.
By virtue of the nonnegativity (due to Assumption \ref{assum:lgm}), the estimate,  $\phi(s,a)^\top\wb{w}^{(i,j)}$, 
of the exact inner product $P(\cdot|{s,a})^{\top} V^{(i,j-1)}$ 
for arbitrary $(s,a)$ is also shifted downwards. 
Then we have 
\[
\phi(s,a)^\top\wb{w}^{(i,j)}\le P(\cdot|{s,a})^{\top} V^{(i,j-1)}\le P(\cdot|{s,a})^{\top} V^{(i,j)}.
\]
By maximizing the lefthandside over $a$, we see that the monotonicity property is preserved inductively. See Lemma \ref{lemma:mono} for a more detailed proof.



\noindent\textbf{Variance reduction.}
The algorithm uses an outer loop and an inner loop for approximately iterating the Bellman operator. 
Each outer iteration performs pre-estimation of a reference vector $PV_{\theta^{(i,0)}}$ (Step 13), which is used throughout the inner loop. 
For instance, let $\theta^{(i,j)}$ be the parameters at outer iteration $i$ and inner iteration $j$.
To obtain an  entry $Q^{(i,j)}(s,a)$ of the new Q-function, we need to estimate 
$P(\cdot|s,a)^\top V_{\theta^{(i,j-1)}}$ with sufficient accuracy, so we have 
\begin{align*}
P(\cdot|s,a)^\top V_{\theta^{(i,j-1)}} 
&= P(\cdot|s,a)^\top (V_{\theta^{(i,j-1)}}  - V_{\theta^{(i,0)}})\\
&\qquad+ P(\cdot|s,a)^\top V_{\theta^{(i,0)}}.
\end{align*}
Note that the reference $P(\cdot|s,a)^\top V_{\theta^{(i,0)}}$ is already approximated with high accuracy in Step 13. This allows the inner loop to successively refine the value and policy, while each inner iteration uses a smaller number of sample transitions to estimate the offset $P(\cdot|s,a)^\top (V_{\theta^{(i,j-1)}}  - V_{\theta^{(i,0)}})$ (Step 18). 

Putting together the preceding techniques, Algorithm \ref{alg:learn} performs carefully controlled Bellman updates so that the estimated value-policy functions monotonically improve to the optimal ones. The algorithm contains $R'=  \Theta(\log[\epsilon^{-1}(1-\gamma)^{-1}])$ many outer loops.
Each outer loop (indexed by $i$) starts with a policy $\|v^* - V_{\theta^{(i,0)}}\|_{\infty}\lesssim H/2^{i}$ and ends with a policy $\|v^* - V_{\theta^{(i+1, 0)}}\|_{\infty}\lesssim H/2^{i+1}$.The algorithm takes multiple rounds of mini-batches, where the sample size of each mini-batch is picked just enough to guarantee the accumulation of total error is within $\epsilon$. The algorithm fully exploits the monotonicity property of the Bellman operator as well as the error accumulation in the Markov process (to be explained later in the proof outline). 

\begin{algorithm*}[htb!]
	\caption{Optimal Phased Parametric Q-Learning (OPPQ-Learning) \label{alg:learn}}
	\begin{algorithmic}[1]
		\State 
		\textbf{Input:} A DMDP $\cM=(\cS,\cA, P, r, \gamma)$ with anchor state-action pairs $\cK$; feature map $\phi:\cS\times \cA \rightarrow \RR$; 
		\State \textbf{Input:} $\epsilon, \delta\in(0,1)$
		\State 
		\textbf{Output:} $\theta \subset \RR^{K}$ with $|\theta| = \Theta[(1-\gamma)^{-1}\log^2\epsilon^{-1}]$
		\State 
		\State
		\textbf{Initialize:}
		$R'\gets  \Theta(\log[\epsilon^{-1}(1-\gamma)^{-1}])$, $R\gets \Theta[R'(1-\gamma)^{-1}]$ 
		\hspace{0.5cm}\codecmt{initialize the numbers of iterations}
		\State 
		\qquad\qquad$\{w^{(i, j)}, \epsilon^{(i,j)}, \wb{w}^{(i, j)}\}_{i\in [0, R'], j\in[0, R]}\subset\RR^{K}$ as $0$ vectors 
		\hspace{0.45cm}\codecmt{initialize parameters}
		\State \qquad\qquad $m\gets  C\cdot \frac{1}{\epsilon^2}\cdot\frac{\log (R'RK\delta^{-1})^{4/3}}{(1-\gamma)^{3}}$,\hspace{3.95cm} ~\codecmt{mini-batch size for outer loop}\\
		 \qquad\qquad $m_1\gets C\cdot \frac{\log (R'RK\delta^{-1})}{(1-\gamma)^{2}}$  for some constant $C$; 
		 \hspace{1.59cm} \codecmt{mini-batch size for inner loop}
		\State \qquad\qquad $\theta^{(0,0)}\gets \{0\}\subset \RR^K$	\hspace{3.5cm} \codecmt{initialize the output to contain a single $0$-vector}
		
		\State 
		\textbf{Iterates:} 
		\vspace{1mm}
		\State\codecmt{Outer loop}
		\For{$i=0,1, \ldots, R'$}
		\For{each $k\in [K]$} 
		\State \label{alg:step-outer-init}Obtain state samples $x_{k}^{(1)}, x_{k}^{(2)}, \ldots, x_{k}^{(m)}\in \cS$ from $P(\cdot|s_k,a_k)$ for $(s_k, a_k)\in \cK$. Let
		\vspace{-3mm}
		 {\small
		\begin{align*}
		w^{(i, 0)}(k) &\gets \frac{1}{m} \sum_{\ell=1}^m
		V_{\theta^{(i,0)}}(x_k^{(\ell)}),
		\qquad
		z^{(i, 0)}(k)
		 \gets \frac{1}{m} \sum_{\ell=1}^m 
		 V_{\theta^{(i,0)}}^2(x_k^{(\ell)})  \\[-0.8em]
		 &\qquad\qquad\qquad\qquad\qquad\qquad\qquad\qquad\qquad\qquad\text{\codecmt{empirical esitimate of $P_{\cK}V_{\theta^{(i,0)}}$ and $P_{\cK}V^2_{\theta^{(i,0)}}$}}
		 \\[-0.3em]
		 \sigma^{(i,0)}(k) &\gets z^{(i, 0)}(k) - ( w^{(i, 0)}(k))^2 \qquad\qquad\qquad\qquad \\[-0.8em]
		 &\qquad\qquad\qquad\qquad\qquad\qquad\qquad\qquad\text{\codecmt{empirical esitimate of variance $P_{\cK}V^2_{\theta^{(i, 0)}} - (P_{\cK}V_{\theta^{(i,0)}})^2$}}\\
		 \epsilon^{(i, 0)}(k) &\gets \Theta\Big[\sqrt{\log[R'RK\delta^{-1}]\cdot\sigma^{(i,0)}(k)\cdot m^{-1}} + \log[R'RK\delta^{-1}](1-\gamma)^{-1}/m^{3/4}\Big]\\[-0.8em]
		 &\qquad\qquad\qquad\qquad\qquad\qquad\text{\codecmt{estimate of the confidence bound of the emprical estimator $w^{(i,0)}$}}
		 \\[-0.3em]
		 \wb{w}^{(i,0)}(k)&\gets \max\big\{0,\quad \min\big\{{w}^{(i,0)}(k) - \epsilon^{(i,0)}(k),\quad (1-\gamma)^{-1}\big\}\big\}\qquad\qquad
		 \text{\codecmt{shift and clip the estimate}}
		\end{align*}
		}
		\vspace{-7mm}
		\EndFor
		\vspace{1mm}
		\State \codecmt{Inner loop}
		\For{$j=1,2,\ldots, R$} 
		\For{each $k\in [K]$}
		\State Obtain state samples $x_{k}^{(1)}, x_{k}^{(2)}, \ldots, x_{k}^{(m_1)}\in \cS$ from $P'(\cdot|s_k, a_k)$ for $(s_k, a_k)\in \cK$. 
		Let
		\vspace{-3mm}
	 	{\small
			\begin{align*}
			w^{(i, j)}(k) &\gets 
			\frac{1}{m_1} \sum_{\ell=1}^m
			\Big(
			V_{\theta^{(i,j-1)}}(x_k^{(\ell)})
			-
			V_{\theta^{(i,0)}}(x_k^{(\ell)}) 
			\Big) + w^{(i,0)}(k) 
			 \hspace{0.5cm}\text{\codecmt{empirical esitimate of $P_{\cK}V_{\theta^{(i,j-1)}}$}}
			\\[-0.3em]
			\epsilon^{(i,j)}(k)& \gets \epsilon^{(i,0)}(k) + \Theta(1-\gamma)^{-1}2^{-i}\sqrt{\log(RR'K\delta^{-1})/m_1}
			 \\[-0.3em]
			 &\qquad\qquad\qquad\qquad\qquad\qquad\qquad\qquad\qquad\text{\codecmt{approximate the confidence bound of $P_{\cK}V_{\theta^{(i,j-1)}}$}}
			\\[-0.3em]
			\wb{w}^{(i,j)}(k)&\gets \max\Big\{0,\quad \min\big\{{w}^{(i,j)}(k) - \epsilon^{(i,j)}(k),\quad (1-\gamma)^{-1}\big\}\Big\}
			\hspace{1.3cm}\text{\codecmt{shift and clip the estimate}}
			\end{align*}}
		\vspace{-7mm}
		\EndFor
		\State $\theta^{(i,j)}\gets \theta^{(i,j-1)}\cup \{\wb{w}^{(i,j)}\}$  \hspace{5.0cm}\codecmt{\small attach the newly estimated parameter to $\theta$}
		\EndFor
		\State $\theta^{(i+1, 0)}\gets \theta^{(i,R)}$  
		\hspace{6.98cm}\codecmt{\small prepare the next outer loop}
		\EndFor
		\State \Return\quad	$\theta^{(R', R)}$
	\end{algorithmic}
\end{algorithm*}


\subsection{Optimal Sample Complexity Guarantee}
In this section, we analyze the sample complexity of the algorithm provided in the last section.
\begin{theorem}[Near-Optimal Sample Complexity\label{thm:main-thm}]
	Suppose $M=(\cS,\cA,P,r,\gamma)$ is an MDP instance admitting the feature representation $\phi:\cS\times\cA\to\RR^K$. Suppose that Assumption~\ref{assum:lgm} holds.
	Let $\delta,\epsilon\in(0,1)$ be parameters.
	Then Algorithm~\ref{alg:learn} takes 
	\[
	N=
	\Theta\bigg[\frac{K}{(1-\gamma)^3 \cdot\epsilon^2}\cdot\log^{4/3} \frac{K}{\epsilon\delta(1-\gamma)}\cdot \log^2\frac{1}{\epsilon(1-\gamma)}\bigg]
	\]	
	samples and outputs $\theta$ 
	such that $\pi_{\theta}$ is $\epsilon$-optimal from every initial state with probability at least $1-\delta$.
\end{theorem}


Theorem \ref{thm:main-thm} is proved through a series of lemmas, which we defer to the appendix. Here we sketch the key ideas.

\begin{proof}[Proof Sketch]



Let $H=\frac1{1-\gamma}$ for short. Each outer-loop iteration decreases the policy error upper bound by at least half.
Suppose $\theta^{(i, 0)}$ is the parameter when the $i$th outer iteration begins, we expect 
$
 \|V_{\theta^{(i, 0)}} - v^*\|_{\infty}\le H/2^{i} ,
$
with high probability. 
Therefore, after  $R' = \log(H/\epsilon) $ iterations, we expect
$
\|V_{\theta^{(R', 0)}} - v^*\|_{\infty}\le H/2^{R'} = O(\epsilon).
$

Now we analyze how many samples are sufficient within one outer-loop iteration. We show that the final error is mainly due to $\epsilon^{(i,0)}$, which comes from estimating the reference function $V_{\theta^{(i,0)}}$ (Line 13). This error is exemplified in the inner loop since $V_{\theta^{(i,0)}}$ is used repeatedly (line 18). 

A key step of the proof is to show that the error contributed by $\epsilon^{(i,0)}$ throughout the inner-loop iterations is small.
By using the monotonicity property, we can show that $$\epsilon^{(i,0)}(s,a)\lesssim \sqrt{\sigma_{v^*}(s,a)/m},\quad\forall~(s,a),$$ 
where $\lesssim$ denotes ``approximately less than" (ignoring non-leading terms), and $\sigma_{v^*}:\cS\times\cA\to\RR$ is an intrinsic variance function of the MDP:
\[
\sigma_{v^*}(s,a) := \var_{s'\sim P(\cdot|s,a)}\big[v^*(s')\big].
\]
By using the monotonicity property, we prove by induction:
\begin{align*}
 v^*&- v^{\pi_{\theta^{(i,R)}}}\le v^*- V_{\theta^{(i,R)}}
\lesssim
\gamma P^{\pi^*}(v^*- V_{\theta^{(i,R-1)}})
 + \epsilon^{(i,0)}_{\pi^*} \\
 &\lesssim\ldots
 \le \gamma^{R}
 (v^*- V_{\theta^{(i,0)}})
 +\sum_{i=0}^{R}\gamma^i(P^{\pi^*})^i\epsilon^{(i,0)}_{\pi^*}
\\
&\lesssim (I-\gamma P^{\pi^*})^{-1} \epsilon^{(i,0)}_{\pi^*}
\\
&\lesssim (I-\gamma P^{\pi^*})^{-1} \sqrt{\sigma_{v^*}^{\pi^*}}/\sqrt{m}, \qquad \text{pointwise w.h.p.}, 
\end{align*}
where $\sigma_{v^*}^{\pi^*}(s) = \sigma_{v^*}(s, \pi^*(s))$, $\epsilon^{(i,0)}_{\pi^*}(s) = \epsilon^{(i,0)}(s, \pi^*(s))$, and $m$ is the mini-batch size.
Now we have found a connection between the error accumulation of the algorithm and the intrinsic variance of the MDP.
By a form of conditional {law of total variance} of the Markov process (Lemma \ref{lemma:var}) and using the convex combination property (Assumption \ref{assum:lgm}), one has
$$ (I-\gamma P^{\pi^*})^{-1} \sqrt{\sigma_{v^*}^{\pi^*}}
= \wt{O}\big(\sqrt{H^3}\big)\cdot \boldsymbol{1}.$$
Therefore the inner loop accumulates error 
$\wt{O}(\sqrt{H^3/m})$, so $m=O(H^3) = O((1-\gamma)^{-3})$ number of samples is enough.

Finally, we  prove by induction that all the desired events happen with sufficiently high probability, so that the iterates improve monotonically to the optimal solution within a sequence of carefully controlled error bars. The total number of outer iterations is nearly constant, therefore the total sample size needed scales with $O((1-\gamma)^{-3})$.
\end{proof}

\begin{remark}[\bf Sample Optimality of Algorithm \ref{alg:learn}]\rm
Theorem \ref{thm:main-thm} matches the information-theoretic lower bound of Theorem \ref{theorem-lower} up to polylog factors with respect to all parameters $S,A,K,\epsilon,1-\gamma$ (note that Theorem~\ref{theorem-lower} still holds under the anchor restriction). Therefore it is a sample-optimal method for solving the feature-based MDP. No other method can outperform it by more than polylog factors. 
\end{remark}

\begin{remark}[\bf About Anchor State-Actions]\rm The proof of Theorem \ref{thm:main-thm} relies on the anchor assumption. The monotonicity property can be preserved because the anchor state-action pairs imply an implicit non-negative factorization of the transition kernel. The convex combination property of anchor state-actions is used in analyzing the error accumulation, needed by the conditional {law of total variance}. Anchor condition is commonly believed to be a key to identifying nonnegative models; see for example \cite{donoho2004does}. We believe this is the first observation that it also relates to sample-optimal reinforcement learning.  

Note that it is possible that the number of anchors is greater than the number of features $K$, then one can append new (dependent) features to make them equal. In this sense Assumption \ref{assum:lgm} {\it always} holds and the actual sample complexity depends on the number of anchors (instead of features). In addition, the anchors can be pre-computed as long as the $\phi$ feature map is known.

\end{remark}

\begin{remark}[\bf Significance of $(1-\gamma)^{-4}$ Improvement] \rm
Let us compare the sample complexities of Algorithms \ref{alg:learn agm}, \ref{alg:learn}.  They differ by a multiplicative gap $(1-\gamma)^{-4}$. Recall that $\gamma\in(0,1)$ is the discount factor. One can view 
$({1-\gamma})^{-1} = 1 +\gamma+\gamma^2+\cdots$ as an approximate horizon.
If $\gamma=0.99$, the MDP essentially has $100$ time steps, and
$$(1-\gamma)^{-4} = 10^8,$$
i.e., Algorithm  \ref{alg:learn} is $10^8$ times faster. It only needs a tiny portion ($1/10^8$) of the samples as needed by the basic algorithm.
We see that clever algorithmic usage of monotonicity and variance structures of the MDP saves big.
\end{remark}

\section{Related Literatures}\label{related-work}
There is a body of works studying the sample complexity of tabular DMDP (i.e., the finite-state finite-action case without structural knowledge). 
Sample-based algorithms for learning value and policy functions have been studied in \citet{kearns1999finite, kakade2003sample, singh1994upper, azar2011speedy, azar2013minimax, sidford2018variance, sidford2018near} and many others.
Among these papers, \citet{azar2013minimax} obtains the first tight sample bound for finding an $\epsilon$-optimal value function,
\citet{sidford2018near} obtains the first tight sample bound for finding an $\epsilon$-optimal policy;
both complexities are of the form  $\wt{O}[|\cS||\cA|(1-\gamma)^{-3}]$.
Lower bounds have been shown in \citet{azar2011reinforcement, even2006action} and \citet{ azar2013minimax}. 
\citet{azar2013minimax} gives the first tight lower bound $\Omega[|\cS||\cA|(1-\gamma)^{-3}]$.


Our result is relevant to the large body of works using linear models and basis functions to approximate value and Q functions. 
For instance,
\citet{tsitsiklis1997analysis, nedic2003least,lagoudakis2003least,melo2008analysis, parr2008analysis, sutton2009convergent, lazaric2012finite, tagorti2015rate} and \citet{maei2010toward} studies  both policy evaluation and optimization by assuming values are from a linear space. 
\citet{tsitsiklis1997analysis} studied the convergence of the temporal-difference learning algorithm for approximating the value function for a fixed policy.
\citet{nedic2003least} studies the {policy evaluation} problem using least square. 
\citet{parr2008analysis} studies the relationships of using linear functions to represent values  and to represent transition models. 
\citet{melo2008analysis} studies the almost sure convergence of $Q$-learning-like methods using linear function approximation.
\citet{sutton2009convergent} shows off-policy temporal-difference learning is convergent with linear function approximation. These earlier works primarily focused on convergence using linear function approximation, without analyzing the sample complexity.

Fitted value iteration (VI) applies to more general function approximators of the value function \cite{ms08, asm08, fsm10, asm08a}, where $v$ is approximated within a low-dimensional function space $\cF$.
They have shown that the error of the fitted-VI is affected by the Bellman error of the space $\cF$.
Their result applies to a general set of functional spaces, where the statistical error depends on a polynomial of $1/\epsilon, 1/(1-\gamma)$ and the intrinsic dimension of the functional space.
It appears that their result works for the $\ell_p$ norm of the policy error, which is proportional to $\epsilon^{-\Theta(p)}$ with high probability. Their result does not apply to the $\ell_{\infty}$ policy error which is the focus of the current paper.

More recently,
\citet{lazaric2012finite, tagorti2015rate} analyzes the sample complexity of temporal difference least square for evaluating a fixed policy.
Recently, a work by \citet{jiang2017contextual} studies the case when a form of Bellman error' is decomposable and has a small rank. 
They show that the number of trajectories needed depends on the Bellman rank rather than the number of states.
\citet{chen2018scalable} proposes a primal-dual method for policy learning that uses linear models and state-action features for both the value and state-action distribution.
To our best knowledge, there is no existing result that solves the linear-model MDP with provable-optimal sample complexity.

\section{Remarks}

The paper studies the information-theoretic sample complexity for solving MDP with feature-based linear transition model. It provides the first sharp sample complexity upper and lower bounds for learning the policy using a generative model. It also provides a sample-optimal parametric Q-learning method that involves computing confidence bounds, variance reduction and monotonic improvement. 
We hope that establishing sharp results for the basic linear model would shed lights on more general structured models and motivate faster solutions.



\section*{Acknowledgment}
We thank Zhuoran Yang for pointing out a flaw in the initial proof of Proposition~\ref{prop-equiv}. We thank the anonymous reviewers for the helpful comments.

\bibliographystyle{plainnat}
\bibliography{ref, reference}


\appendix
\onecolumn

\section{Proofs of Propositions 1,2}

\begin{proof}[Proof of Proposition~\ref{prop:relatelm}]
Let $v^{\pi}$ be the value function of $\pi$.
Since $M \in \mathcal{M}^{trans}(\cS,\cA,\gamma, \phi)$, we have $P(s'|s,a) = \sum_{k\in [K]}\psi_{k}(s') \phi_k(s,a)$ for some $\psi_k$'s.  We have
\begin{align*}
Q^{\pi}(s,a) &= r(s,a) + \gamma \sum_{s'\in \cS}P(s'|s,a) v^{\pi}(s')
=r(s,a) + \gamma \sum_{k\in [K]} \phi_k(s,a) \sum_{s'\in \cS}\psi_{k}(s') v^{\pi}(s')\\
&=r(s,a) + \gamma \sum_{k\in [K]}\phi_k(s,a)w^\pi(k)
\end{align*}
where vector $w^{\pi}\in \RR^{K}$ is specified by
\[
\forall k\in [K]: 
w^\pi(k) = \sum_{s'\in\cS} \psi_k(s') v^{\pi}(s').
\]
Therefore $Q^{\pi} \in \hbox{Span}(r,\phi).$
\end{proof}

\begin{proof}[Proof of Proposition~\ref{prop-equiv}]
``If'' direction: 
Since $M\in \cM^{trans}$, we have from the proof of Proposition~\ref{prop:relatelm} that for any $Q\in \cF$, $\cT Q\in \cF$.

``Only if'' direction: If $d(\cT \cF, \cF) = 0$, then 
for any $Q\in \cF$  
We have
\[
\cT Q = r + \gamma P V(Q) \in \cF.
\]
We can then pick a maximum-sized set $\{Q_1, Q_2, \ldots Q_k
\}\subset \cF$ such that 
$V(Q_1), V(Q_2), \ldots V(Q_k)$ are linear independent.
Note that $k\le K$.
Denote $A = [V(Q_1), V(Q_2), \ldots V(Q_k)]$, $B=[\cT Q_1, \cT Q_2, \ldots, \cT Q_k]$ and $R=[r,r,r\ldots, r]$ (with $k$ columns).
We then have
\[
B =  R + \gamma P A. 
\]
Hence we have
\[
P = \gamma^{-1}(B-R)A^\top (AA^\top)^{-1}. 
\]
Since each column of $B-R$ is a vector in $\cF$, we conclude that each column of $P$ is a vector in $\cF$.
\end{proof}


\section{Proof of Theorem \ref{theorem-lower}}
\begin{proof}[Proof of Theorem~\ref{theorem-lower}]
Let $\cM'$ be the class of all tabular DMDPs with state space $\cS'$, action space $\cA'$, and discount factor $\gamma$.
Let $\cK'$ be an algorithm for such a class of DMDPs with a generative model.
Let \[
N=O\bigg(\frac{|\cS'||\cA'|}{(1-\gamma)^3\cdot \epsilon^2\cdot \log \epsilon^{-1}}\bigg).
\]	
For each $M'\in \cM'$, let $\pi^{\cK', M', N}$ be the policy  returned by $\cK'$ with querying at most $N$ samples from the generative model.
The lower bound in Theorem~B.3 in \citet{sidford2018near}(which is derived from Theorem~3 in \citet{azar2013minimax}) states that
\begin{align*}
\inf_{\cK'}\sup_{M'\in \cM'} \mathbb{P}\bigg[ \sup_{s\in\cS} (v^{*, M'}(s)-v^{\pi^{\cK', M',N}}(s) )&\ge \epsilon  \bigg] \geq 1/3,
\end{align*}
where $v^{*, M'}$ is the optimal value function of $M'$.
Suppose, without loss of generality, $K = |\cS'||\cA'| + 1$.
We prove Theorem~\ref{theorem-lower} by showing that every DMDP instance $M'\in \cM'$ can be converted to an instance $M\in \cM^{trans}_K(\cS,\cA, \gamma)$ such that any algorithm  $\cK$ for $\cM^{trans}_K(\cS,\cA, \gamma)$ can be used to solve $M'$.


For a DMDP instance $M' = (\cS', \cA', P', r', \gamma)\in \cM'$, we construct a corresponding DMDP instance $M=(\cS, \cA, P, r, \gamma)\in\cM^{trans}_K(\cS,\cA, \gamma)$ with a feature representation $\phi$.
We pick $\cS$ and $\cS$ to be supersets of $\cS$ 
and $\cA'$  respectively, so that the transition distributions and rewards remain unchanged on $\cS'\times \cA'$, i.e., $P(\cdot\mid s,a)=P'(\cdot\mid s,a)$ and $r( s,a)=r'( s,a)$  for $s\in\cS',a\in\cA'$.
From $(s,a)\in (\cS\times \cA)/(\cS'\times \cA')$, the process transitions to an absorbing state $s^0\in \cS/\cS'$ with probability 1 and stays there with reward 0.

Now we show that $M$ admits a feature representation $\phi: \cS\times\cA \to\RR^K$ as follows. Say $(s,a)$ is the $k$-th element in $\cS'\times\cA$, we let $\phi(s,a) = \bf{1}_k$, which is the unit vector whose $k$th entry equals one. For $(s,a)\notin \cS'\times\cA'$, we let $\phi(s,a) = \bf{1}_K$.
Then we can verify that $P(s'\mid s,a) = \sum_{k\in[K]} \phi_k(s,a) \psi_k(s')$ for some $\psi_k$'s. Thus we have constructed an MDP instance $M'\in\cM^{trans}_K(\cS,\cA, \gamma)$ with feature representation $\phi$ . 


Suppose that $\cK$ is an algorithm that applies to $M$ using $N$ samples. Based on the reduction, we immediately obtained an algorithm $\cK'$ that applies to $M'$ using $N$ samples and the feature map $\phi$: $\cK'$ works by applying $\cK$ to $M$ and outputs the restricted policy on $\cS'\times\cA'$.
It can be easily verified that if $\pi$ is an $\epsilon$-optimal policy for $M$ then the reduction gives an $\epsilon$-optimal policy for $M'$. 
By virtue of the reduction, one gets
\begin{align*}
\inf_{\cK}\sup_{M\in \cM^{trans}_K(\cS,\cA, \gamma)} \mathbb{P}\bigg( \sup_{s\in\cS} (v^*(s)-v^{\pi^{\cK, M,N}}(s) )\ge \epsilon  \bigg) 
&\ge 
\inf_{\cK'}\sup_{M'\in \cM'} \mathbb{P}\bigg( \sup_{s\in\cS} (v^{*, M'}(s)-v^{\pi^{\cK', M',N}}(s) )\ge \epsilon  \bigg)\\
&\geq 1/3,
\end{align*}
This completes the proof.
\end{proof}

\section{Proof of Theorem \ref{thm:agm-result}.}

\begin{proof}
	Recall that ${P}_{\cK}$ is a submatrix of ${P}$ formed by the rows indexed by $\cK$.
	We denote $\wt{P}_{\cK}$ in the same manner for $\wt{P}$.
	Recall that $\|P - \wt{P}\|_{1,\infty}\le \xi$. 
	Let $\wh{P}_{\cK}^{(t)}$ be the matrix of empirical transition probabilities based on $m:=N/(KR)$ sample transitions per $(s,a)\in\cK$ generated at iteration $k$. It can be viewed as an estimate of ${P}_{\cK}$ at iteration $t$.
	Since $\wt{P}$ admits a context representation, 
	it can be written as 
	\[
	\wt{P} = \Phi\Psi\quad\text{where}\quad \Psi = \Phi_{\cK}^{-1}\wt{P}_{\cK}.
	\]

	Let $\wh{\Psi}^{(t)}=\Phi_{\cK}^{-1}\wh{P}_{\cK}^{(t)}$ be the estimate of $\Psi$ at iteration $t$.
	We can view $\Phi\wh{\Psi}^{(t)}$ as an estimate of $P$.
	
	We will show that each iteration of the algorithm is an approximate value iteration.
	We first define the approximate Bellman operator, $\wh{\cT}$ as, $\forall v\in \RR^{\cS}:\quad$ 
	\begin{align*}
	[\wh{\cT}^{(t)} v](s) 
= \max_a\Big[r(s,a) + \gamma \phi(s,a)^\top \Phi_{\cK}^{-1}\wh{P}_{\cK}^{(t)} v\Big].
	\end{align*}
	Notice that, by definition of the algorithm,
	\[
	V_{w^{(t)}} \gets \wh{\cT}^{(t)} \Pi_{[0,H]}[V_{w^{(t-1)}}],
	\]
	where $w^{(0)} = 0\in \RR^{K}$ and $w^{(t)}$ is the $w$ at the end of the $t$-th iteration of the algorithm and $H=(1-\gamma)^{-1}$ and $\Pi_{[0,H]}(\cdot)$ denotes entrywise projection to $[0,H]$.
	For the rest of the proof, we denote
	\[
	\wh{V}_{w^{(t-1)}} = \Pi_{[0,H]}[V_{w^{(t-1)}}].
	\]
	
	We now show the approximation quality of $\wh{\cT}$,
	i.e., estimate $\| \wh{\cT}^{(t)} \wh{V}_{w^{(t-1)}} -  {\cT} \wh{V}_{w^{(t-1)}}\|_{\infty}$, where $\cT$ is the exact Bellman operator.
	Notice that
	\begin{align*}
	\forall s:\quad |[\wh{\cT}^{(t)}  \wh{V}_{w^{(t-1)}}](s) - 
	[\cT  \wh{V}_{w^{(t-1)}}](s)|
	&\le \gamma\max_{a}\big| \phi(s,a)^\top \Phi_{\cK}^{-1}  \wh{P}_{\cK}^{(t)}  \wh{V}_{w^{(t-1)}} - 
	P(\cdot|s,a)^\top  \wh{V}_{w^{(t-1)}}\big|.
	\end{align*}
	It remains to show the right hand side of the above inequality is small.

	Denote $\cF_t$ to be the filtration defined by  the samples up to iteration $t$.
		Then, by the Hoeffding inequality and the fact that the samples at iteration $t$ are independent with that from iteration $t-1$, we have
	\begin{align*}
	\Pr\bigg[\|\wh{P}_{\cK}^{(t)} \wh{V}_{w^{(t-1)}} &- P_{\cK} \wh{V}_{w^{(t-1)}}\|_{\infty}\le \epsilon_1
	\bigg|\cF_{t-1}\bigg]\ge 1-\delta/R
	\end{align*}
	where we denote 
	\[\epsilon_1 = cH\cdot\sqrt{\frac{\log(KR\delta^{-1})}{m}}
	\]
	for 
	some generic constant $c$.
	Next,	
	let $\cE_t$ be the event that,
	\begin{align*}
	\| \wh{P}_{\cK}^{(t)} \wh{V}_{w^{(t-1)}} -  P_{\cK}  \wh{V}_{w^{(t-1)}}\|_{\infty} 
	 \le \epsilon_1.
	\end{align*}
	We thus have $\Pr[\cE_{t}|\cF_{t-1}] \ge 1-\delta/R$ and $\Pr[\cE_t|\cE_{1}, \cE_2, \ldots\cE_{t-1}] \ge 1-\delta/R$ since $\cE_{1}, \cE_2, \ldots\cE_{t-1}$ are adapted to $\cF_{t-1}$.
	This lead to 
	\[
	\Pr[\cE_1\cap \cE_2 \cap \ldots \cap \cE_R]=
	\Pr[\cE_1]\Pr[ \cE_2 |\cE_1] \ldots \ge 1-\delta.
	\]
	Now we consider event $\cE := \cE_1\cap \cE_2 \cap \ldots \cap \cE_R$, on which we have, for all $t\in [R]$,
	\begin{align*}
	|\phi(s,a)^\top \Phi_{\cK}^{-1} & \wh{P}_{\cK}^{(t)} \wh{V}_{w^{(t-1)}} -\phi(s,a)^\top \Phi_{\cK}^{-1} P_{\cK}  \wh{V}_{w^{(t-1)}}|
	\le \|\phi(s,a)^\top \Phi_{\cK}^{-1}\|_1\cdot\epsilon_1\le L\epsilon_1.
	\end{align*}
	Note that, $\|{P}_{\cK}-\wt{P}_{\cK}\|_{1,\infty}\le \xi$, we thus have
	\begin{align*}
	|\phi(s,a)^\top &\Phi_{\cK}^{-1}  \wh{P}_{\cK}^{(t)} \wh{V}_{w^{(t-1)}} -
	\phi(s,a)^\top \Phi_{\cK}^{-1} \wt{P}_{\cK}  \wh{V}_{w^{(t-1)}}|
	\le L\epsilon_1 + |\phi(s,a)^\top\Phi_{\cK}^{-1}(P_{\cK} - \wt{P}_{\cK})\wh{V}_{w^{(t-1)}}|
	\le L\epsilon_1 + LH\xi,
	\end{align*}
	Further using 
	\[|(\phi(s,a)^\top\Phi_{\cK}^{-1}  \wt{P}_{\cK}^{(t)} - P(\cdot|s,a)^\top) \wh{V}_{w^{(t-1)}}|\le H\xi,
	\]
	we thus have
	\begin{align*}
	|\phi(s,a)^\top \Phi_{\cK}^{-1}  \wh{P}_{\cK}^{(t)} \wh{V}_{w^{(t-1)}} -
	{P}(\cdot|s,a)^\top  \wh{V}_{w^{(t-1)}}|
	&\le |\phi(s,a)^\top( \Phi_{\cK}^{-1}  \wh{P}_{\cK}^{(t)} - \Phi_{\cK}^{-1}  \wt{P}_{\cK}^{(t)} +  \Phi_{\cK}^{-1}  \wt{P}_{\cK}^{(t)}) \wh{V}_{w^{(t-1)}}\\
	&\qquad-
	{P}(\cdot|s,a)^\top  \wh{V}_{w^{(t-1)}}|\\
	&\le L\epsilon_1 + LH\xi + H\xi.
	\end{align*}
	Further notice that $\Pi_{[0, H]}$ can only makes error smaller.
	Therefore, we have shown that the $\wh{V}_{w^{(t)}}$s follow an approximate value iteration with error $\gamma [L\epsilon_1 + (L+1)H\xi]$ with probability at least $1-\delta$.
	Because of the contraction of the operator $\cT$, we have, after $R$ iterations, 
	\begin{align*}
	\|\wh{V}_{w^{(R-1)}}- v^*\|_{\infty}  &\le \gamma^{R-1}H 
	+ \gamma R [L\epsilon_1 + (L+1)H\xi]\le \gamma R [2L\epsilon_1 + (L+1)H\xi]
	\end{align*}
	for appropriately chosen $R=\Theta(\log (NH)/(1-\gamma))$.
	Since $Q_{w^{(R)}}(s,a) = r(s,a) + \gamma \phi(s,a)^\top \Phi_{\cK}^{-1}\wh{P}_{\cK}^{(R)}\wh{V}_{w^{(R-1)}}$, we have,
	\[
	\|Q_{w^{(R)}} - Q^*\|_{\infty} \le 2\gamma R [2L\epsilon_1 + (L+1)H\xi]
	\]
	happens with probability at least $1-\delta$.
	It follows that (see, e.g., Proposition~{2.1.4} of \cite{bertsekas2005dynamic}),
	\[
	\|v^{\pi_{w^{(R)}}} - v^*\|_{\infty}\le 2\gamma RH  [2L\epsilon_1 + (L+1)H\xi],
	\]
	with probability at least $1-\delta$.
	Plugging the values of $H, \epsilon_1$ and $m$, we have
	\[
	\|v^{\pi_{w^{(R)}}} - v^*\|_{\infty} 
	\le C\gamma\cdot \frac{\log (NH)}{1-\gamma} \cdot \frac{1}{1-\gamma} \cdot L \cdot  \sqrt{\frac{K\log(KR\delta^{-1})}{(1-\gamma)^2\cdot N}\cdot 
	 \frac{\log (NH)}{1-\gamma}}
 	+C\gamma\cdot \frac{\log (NH)}{1-\gamma} \cdot \frac{L}{(1-\gamma)^2} \cdot \xi
	\]
	for some generic constant $C > 0$.
	This completes the proof.
%

\end{proof}

\section{Proof of Theorem \ref{thm:main-thm}}
According to the discussions following Assumption \ref{assum:lgm}, we assume without loss of generality:
\begin{itemize}
\item For each anchor $(s_k,a_k)\in \mathcal{K}$, $\phi(s_k,a_k)$ is a  vector with $\ell_1$-norm $1$.
\end{itemize}
Then Assumption 2 further implies
\begin{itemize}
\item $\phi(s,a)$ is a vector of probabilities for all $(s,a)$.
\item For each $(s,a)$, $P(\cdot | s,a) =\sum_k \phi_k(s,a) P(\cdot\mid s_k,a_k)$. 
\end{itemize}

\subsection{Notations}

\noindent\textbf{$\cT$-operator}
For any value function $V:\cS\rightarrow \RR$ and policy $\pi:\cS\rightarrow \cA$,
we denote the Bellman operators as  
\[
\cT V[s] =
\max_{a\in \cA} \big[r(s,a) + \gamma P(\cdot|s,a)^\top V\big]
\quad\text{and}\quad
\cT_\pi V[s] =
r(s,\pi(s)) + \gamma P(\cdot|s,\pi(s))^\top V
\]
The key properties, e.g. monotonicity and contraction, of the $\cT$-operator can be found in \citet{puterman2014markov}. 
For completeness, we state them here.
\begin{fact}[Bellman Operator]
For any value function $V, V':\cS\rightarrow \RR$, if $V\le V'$ entry-wisely, we then have,
\begin{align*}
\cT V\le \cT V'&\quad\text{and}\quad
\cT_{\pi} V\le \cT_{\pi} V', \\ 
\|\cT V-v^*\|_{\infty}\le \gamma \| V - v^*\|_{\infty}&\quad\text{and}\quad \|\cT_{\pi} V-v^{\pi}\|_{\infty}\le \gamma \| V - v^{\pi}\|_{\infty},\\
\lim_{t\rightarrow\infty}\cT^t V = v^*&\quad\text{and}\quad
\lim_{t\rightarrow\infty}\cT_{\pi}^t V = v^{\pi}.
\end{align*}
\end{fact}

\paragraph{$Q$-function} We let, for any $(s,a)$,
\begin{align*}
Q_{\theta^{(i, j)}}(s, a) &= 
r(s, a) + \gamma \phi(s,a)^{\top} \wb{w}^{(i,j)},
\\ 
\wb{Q}_{\theta^{(i, j)}}(s, a) &= r(s, a) + \gamma P(\cdot|s,a)^\top V_{\theta^{(i,j - 1)}}(\cdot).
\end{align*}
\paragraph{Variance of value function}
For $(s,a)$, we denote the variance of a function (or a vector) $
V:\cS\rightarrow \RR$ as,
\[
\sigma_{s,a}[V]
:= \sum_{s'}P(s'|s,a)V^2(s')
- \Big(\sum_{s'}P(s'|s,a)V(s')\Big)^2,
\]
we also denote $\sigma_{k} (\cdot)= \sigma_{s_k,a_k}(\cdot)$ for $(s_k, a_k)\in \cK$.
\paragraph{$\cE$-event}
In Algorithm~\ref{alg:learn},
let $\cE^{(i,0)}$ be the event that 
\begin{align*}
\forall k\in [K]: |w^{(i,0)}(k) -  P(\cdot|s_k, a_k)^\top V_{\theta^{(i,0)}}|&\le
\epsilon^{(i,0)}(k)
 \le C\bigg[\sqrt{\frac{\log(R'RK\delta^{-1})\sigma_k[V_{\theta^{(i, 0)}}]}{m}} + 
 \frac{\log(R'RK\delta^{-1})}{(1-\gamma)m^{3/4}}
\bigg]
\end{align*}
for some generic constant $C>0$.
Let $\cE^{(i,j)}$ be the event on which
\begin{align*}
\forall k\in [K]:\quad 
&|w^{(i,j)}(k) - w^{(i,0)}(k) -P(\cdot|s_k,a_k)^\top (V_{\theta^{(i, j-1)}} - V_{\theta^{(i, 0)}})|\le C(1-\gamma)^{-1}2^{-i}\sqrt{\log(R'RK\delta^{-1})/m_1},
\end{align*}
where $R', R, m, m_1$ are parameters defined in Algorithm~\ref{alg:learn}.

\paragraph{$\cG$-event}
Let $\cG^{(i)}$ be the event such that
$$0\le V_{\theta^{(i,0)}}(s)\le \cT_{\pi_{\theta^{(i, 0)}}} V_{\theta^{(i, 0)}}[s]\le v^*(s)
,\qquad
v^*(s)-V_{\theta^{(i,0)}}(s)\le c2^{-i}/(1-\gamma),\qquad\forall s\in\cS,$$ for some sufficiently small constant $c$.

\subsection{Some Properties}
Firstly we notice that the parameterized functions $Q_{\theta}, V_{\theta}$ (eq. \eqref{eqn:vpidecoder1}) increase pointwisely (as index $(i,j)$ increases). 
\begin{lemma}[Monotonicity of the Parametrized $V$]
	\label{lemma:mon}
For every $(i, j), (i',j') \in [R']\times[R]$, and $s\in \cS$, if $(i, j)\le (i',j')$ (in lexical order), we have
	\[
	V_{\theta^{(i, j)}}(s)  \le V_{\theta^{(i', j')}}(s).
	\]
\end{lemma}

We note the triangle inequality of variance.
\begin{lemma}
	For any $V_1, V_2:\cS\rightarrow \RR$, we have
	$\sqrt{\sigma_{k}[V_1+V_2]}\le \sqrt{\sigma_{k}[V_1]}+\sqrt{\sigma_k[V_2]}$ for all $k\in [K]$.
\end{lemma}

The next is a key lemma showing a property of the convex combination of the standard deviations, which relies on the anchor condition.
\begin{lemma}
	\label{lemma:var}For any $V:\cS\rightarrow\RR$ and $s,a\in \cS\times\cA$:
	\[
	 \sum_{k\in [K]}\phi_k(s,a)\sqrt{\sigma_{k}[V]}\le \sqrt{\sigma_{s,a}(V)}.
	\]
\end{lemma}
\begin{proof}
	Since $[\phi_1(s,a),\ldots,\phi_K(s,a)]$ is a vector of probability distribution (due to Assumption 2 without loss of generality), by Jensen's inequality we have,
	{\small
	\begin{align*}
	\sum_{k}\phi_k({s,a})\sqrt{\sigma_{k}[V]}\le \sqrt{\sum_{k}\phi_k({s,a})\sigma_{k}[V]} &= \sqrt{\sum_{k}\phi_k({s,a})\bigg[\sum_{s'}P(s'|s_k,a_k)V^2(s')
		- \Big(\sum_{s'}P(s'|s_k,a_k)V(s')\Big)^2\bigg]}\\
	&=\sqrt{\sum_{s'}P(s'|s,a)V^2(s')
		- \sum_{k}\phi_k({s,a})\bigg[\Big(\sum_{s'}P(s'|s_k,a_k)V(s')\Big)^2\bigg]}.
	\end{align*}}
	By the Jensen's inequality again, we have
	\begin{align*}
	\sum_{k}\phi_k({s,a})&\Big(\sum_{s'}P(s'|s_k,a_k)V(s')\Big)^2 \ge 
	\Big(\sum_{k}\phi_k({s,a})\sum_{s'}P(s'|s_k,a_k)V(s')\Big)^2= \Big(\sum_{s'}P(s'|s,a)V(s')\Big)^2.
	\end{align*}
	Combining the above two equations, we complete the proof.
\end{proof}


\subsection{Monotonicity Preservation}

The next lemma illustrates, conditioning on $\cE^{(i,j)}$ and $\cG^{(i)}$, a monotonicity property is preserved throughout the inner loop.
\begin{lemma}[Preservation of Monotonicity Property]
	\label{lemma:mono}
	Conditioning on the events $\cG^{(i)}$, $\cE^{(i,0)}, \cE^{(i,1)}, \ldots, \cE^{(i,j)}$, we have
	for all
	$s\in \cS, j'\in [0, j]$, 
	\begin{align}
	V_{\theta^{(i, j')}}(s)\le \cT_{\pi_{\theta^{(i, j')}}} V_{\theta^{( i,j')}}[s]\le \cT V_{\theta^{ (i,j')}}[s]\le v^*(s).\label{eqn:ind}
	\end{align}
	Moreover,
	for any fixed policy $\pi^*$, we have, for $j'\in[j] $, 
	\begin{align}\label{eqn:indd}
	v^{*}(s) - V_{\theta^{( i, j')}}(s)\le &\gamma P(\cdot|s,\pi^*(s))^\top(v^* -V_{\theta^{( i, j'-1)}})+ 2\gamma \sum_{k}\phi_k(s,\pi^*(s))\epsilon^{(i,j')}(k).
	\end{align}
\end{lemma}
\begin{proof}
\textbf{\\
	Proof of \eqref{eqn:ind} by Induction}: 
	We first prove the inequalities in \eqref{eqn:ind} 
	by induction on $j'$.
	The base case of $j'=0$ holds by definition of $\cG^{(i)}$.
	
	Now assuming it holds for $j'-1\ge 0$, let us verify that \eqref{eqn:ind} holds for $j'$.
	For any $s\in \cS$, we rewrite the corresponding value function defined in \eqref{eqn:vpidecoder1} as follows:
	\[
	V_{\theta^{(i, j')}}(s)  = \max\big\{\max_{a}Q_{\theta^{(i,j')}}(s,a), V_{\theta^{(i, j'-1)}}(s)\big\}.
	\] 
	For any $s\in \cS$, there are only two cases to make the above equation hold:
		\begin{enumerate}
			\item $V_{\theta^{(i, j')}}(s) = V_{\theta^{(i, j'-1)}}(s)$ $\Rightarrow$ $\max_{a}Q_{\theta^{(i,j')}}(s,a) < V_{\theta^{(i, j'-1)}}(s)$ and $\pi_{\theta^{(i,j')}}(s) = \pi_{\theta^{(i,j'-1)}}(s)$; 
		\item $V_{\theta^{(i,j')}}(s) = \max_{a}Q_{\theta^{(i,j')}}(s,a)$  $\Rightarrow$ $\max_{a}Q_{\theta^{(i,j')}}(s,a) \ge V_{\theta^{(i, j'-1)}}(s)$ and $\pi_{\theta^{(i,j')}}(s) = \arg\max_{a}Q_{\theta^{(i,j')}}(s,a)$.
	\end{enumerate}
	We investigate the consequences of case 1. 
	Since  \eqref{eqn:ind} holds for $j'-1$,
	we have $V_{\theta^{(i, j')}}(s)=V_{\theta^{(i, j'-1)}}(s)\le v^*(s)$. Moreover, since  \eqref{eqn:ind} holds for $j'-1$ and $\pi_{\theta^{(i,j')}}(s) = \pi_{\theta^{(i,j'-1)}}(s)$, we have
	\begin{align*}
	V_{\theta^{(i, j')}}(s)=V_{\theta^{(i, j'-1)}}(s)&\le \cT_{\pi_{\theta^{(i,j')}}} V_{\theta^{(i, j'-1)}}[s] \hspace{25mm}\text{\codecmt{by induction hypothesis}}\\
	&\le 
	\cT_{\pi_{\theta^{(i,j')}}} V_{\theta^{(i, j')}}[s]
	\hspace{28mm}\text{\codecmt{by Lemma~\ref{lemma:mon} and the monotonicity of $\cT_{\pi}$}}\\
	&\le \cT V_{\theta^{(i, j')}}[s].
	\end{align*}
	We now investigate the consequences of case 2. 
	Notice that conditioning on $\cE^{(i,0)}, \cE^{(i,1)}\ldots, \cE^{(i,j')}$ (by specifying the constant $C$ appropriately), we can verify that, 
	\begin{align*}
	\forall k\in [K]: \quad &\wb{w}^{(i, j')}(k):=\Pi_{[0, H]}(w^{(i, j')}(k) - \epsilon^{(i, j')}(k))
	\le P(\cdot|s_k,a_k)^\top V_{\theta^{(i, j'-1)}},
	\end{align*}
	where $H=(1-\gamma)^{-1}$.
	Thus, for any $a\in\cA$,
	\begin{align*}
	Q_{\theta^{( i, j')}}(s,a) &= r(s,a) + \gamma \phi({s,a})^\top \wb{w}^{(i, j')}
	\le r(s,a) + \gamma \sum_{k\in [K]}\phi_k(s,a) P(\cdot|s_k,a_k)^\top V_{\theta^{(i, j'-1)}}
	= \wb{Q}_{\theta^{(i, j')}}(s,a).
	\end{align*}
	Then we have 
	\begin{align}
	0\le \max_{a}Q_{\theta^{(i,j')}}(s,a)&={Q}_{\theta^{(i,j')}}(s,\pi_{\theta^{(i,j')}}(s));\nonumber\\
	\max_{a}Q_{\theta^{(i,j')}}(s,a)\le\wb{Q}_{\theta^{(i,j')}}(s,\pi_{\theta^{(i,j')}}(s))
	&= \cT_{\pi_{\theta^{(i,j')}}} V_{\theta^{(i,j'-1)}}[s];\nonumber\\
	\max_{a}Q_{\theta^{(i,j')}}(s,a)
	\le \max_{a}\wb{Q}_{\theta^{(i,j'-1)}}(s,a)
	&= \cT V_{\theta^{(i,j'-1)}}[s].
	\label{eqn:understeimate-Q}
	\end{align}
	As a result, we obtain
	\begin{align*}
	0&\le V_{\theta^{(i, j')}}(s) = \max_{a}Q_{\theta^{(i,j')}}(s,a) \le \cT_{\pi_{\theta^{(i,j')}}(s)} V_{\theta^{(i, j'-1)}}[s]\\
	&\le \cT_{\pi_{\theta^{(i,j')}}} V_{\theta^{(i, j')}}[s]
	\hspace{20mm}\text{\codecmt{by Lemma~\ref{lemma:mon} and the monotonicity of $\cT_{\pi}$}}\\
	&\le \cT V_{\theta^{(i, j')}}[s]. 
	\end{align*}
	We see that $	0\le V_{\theta^{(i, j')}}(s) 
	\le \cT_{\pi_{\theta^{(i,j')}}} V_{\theta^{(i, j')}}[s]
	\le \cT V_{\theta^{(i, j')}}[s] $
	holds in both cases 1 and 2.
	Also note that since \eqref{eqn:ind} holds for $j'-1$, we have $V_{\theta^{(i, j'-1)}}\le v^*$. It follows from the monotonicity of the Bellman operator that
	\[
	0\le V_{\theta^{(i, j')}}(s)\le \cT_{\pi_{\theta^{(i,j')}}} V_{\theta^{(i, j'-1)}}[s]\le 
	\cT_{\pi_{\theta^{(i,j')}}} v^*[s] \le v^*(s).
	\]
	This completes the induction.

	\textbf{\\
	Proof of \eqref{eqn:indd}}: 	
	Let $\pi^*$ be some fixed optimal policy.
	For each $j'\in [j]$, by \eqref{eqn:vpidecoder1}, we have
	\[
	V_{\theta^{(i,j')}}(s) \ge \max_{a\in \cA} Q_{\theta^{(i,j')}}(s,a):= \max_{a\in \cA}\big[r(s, a) + \gamma \phi(s,a)^{\top} \wb{w}^{(i,j')}\big]. 
	\]
	By definition of $\cE^{(i,j')}$, we have
	\[
	\forall k\in [K]:\quad \wb{w}^{(i,j')}(k) \ge w^{(i,j')}(k) - \epsilon^{(i,j')}(k) \ge P(\cdot|s_k,a_k)^\top V_{\theta^{( i, j'-1)}} - 2\epsilon^{(i,j')}(k).
	\]
	Therefore,
	\begin{align*}
	V_{\theta^{(i, j')}}(s)&\ge 
	\max_a\Big[r(s,a) + \gamma \sum_{k}\phi_k(s,a) \big(P(\cdot|s_k,a_k)^\top V_{\theta^{( i, j'-1)}} - 2\epsilon^{(i,j')}(k)\big)\Big].
	\end{align*}
	Hence,
	{\small
	\begin{align*}
	v^{*}(s) - V_{\theta^{(i, j')}}(s) &\le  r^{\pi^*}(s) + \gamma P^{\pi^*}(\cdot|s)^\top v^{*}  - \max_a\Big[r(s,a) + \gamma \sum_{k}\phi_{k}(s,a) \big(P(\cdot|s_k,a_k)^\top V_{\theta^{(i, j'-1)}} - 2\epsilon^{(i,j')}(k)\big)\Big]\\
	& \le r^{\pi^*}(s) + \gamma P^{\pi^*}(\cdot|s)^\top v^{*}   - \Big[r(s,\pi^*(s)) + \gamma \sum_{k}\phi_{k}(s,\pi^*(s)) \big(P(\cdot|s_k,a_k)^\top V_{\theta^{( i, j'-1)}} - 2\epsilon^{(i,j')}(k)\big)\Big] \\
	&= \gamma P^{\pi^*}(\cdot|s)^\top(v^* -V_{\theta^{(i, j'-1)}})
	+ 2\gamma \sum_{k}\phi_{k}(s,\pi^*(s))\epsilon^{(i,j')}(k),
	\end{align*}}
	where $P^{\pi^*}(\cdot|s) = P(\cdot|s, \pi^*(s))$ and we use the fact that
	$P^{\pi^*}(\cdot|s) = \sum_{k}\phi_{k}(s,\pi^*(s)) P(\cdot|s_k,a_k)$ in the last equality. 
\end{proof}

\subsection{Accuracy of Confidence Bounds}
We show that the mini-batch sample sizes picked in Algorithm 2 are sufficient to control the error occurred in the inner-loop iterations, such that the events $\cE^{(i,0)}, \cE^{(i,1)}, \ldots, \cE^{(i,R)}$ jointly happen with close-to-1 probability.

\begin{lemma}\label{lemma-9}For $i=0, 1, 2, \ldots, R'$,
	\[
	\PP[\cE^{(i,0)}, \cE^{(i,1)}, \ldots, \cE^{(i,R)}|\cG^{(i)}]
	\ge 1-\delta/R'.
	\]
\end{lemma}

\begin{proof} We analyze each event separately.

\textbf{\\Probability of $\cE^{(i,0)}$}: We first show that $\Pr[\cE^{(i,0)}|\cG^{(i)}]\ge 1-\delta/(RR')$.
	Note that $V_{\theta^{(i, 0)}}(s)\in[0,\frac1{1-\gamma}]$ is determined by the samples obtained before the outer-iteration $i$ starts, therefore samples obtained in iteration $(i,j)$ for $j\ge0$ are independent with $V_{\theta^{(i, 0)}}$.
	Hence, conditioning on $\cG^{(i)}$, for a fixed $\delta\in(0,1)$ and $k\in [K]$, by the Bernstein's and the Hoeffding's inequalities, for some constant $c_1>0$, the following two inequalities hold with probability at least $1-\delta$,
	\begin{align*}
	\bigg|w^{(i,0)}(k) - P(\cdot|s_k,a_k)^\top V_{\theta^{(i, 0)}}\bigg|
	&\le \min\bigg\{c_1 \sqrt{\frac{\log[\delta^{-1}]\sigma_k[V_{\theta^{(i, 0)}}]}{m}}  + \frac{c_1\log\delta^{-1}}{(1-\gamma)m},\quad 
	c_1(1-\gamma)^{-1}\cdot \sqrt{\frac{\log[\delta^{-1}]}{m}}
	\bigg\}
	\\
	\bigg|z^{(i,0)}(k) - P(\cdot|s_k,a_k)^\top V_{\theta^{(i, 0)}}^2\bigg|
	&\le c_1 (1-\gamma)^{-2}\cdot \sqrt{\frac{\log[\delta^{-1}]}{m}}
	,
	\end{align*}
	where we recall the notation $\sigma_k[V_{\theta^{(i, 0)}}] =  P(\cdot|s_k,a_k)^\top V_{\theta^{(i, 0)}}^2 - [P(\cdot|s_k,a_k)^\top V_{\theta^{(i, 0)}}]^2 \le (1-\gamma)^{-2}$ (see D.1).
	Conditioning on the preceding two inequalities, we have
	\[
	\bigg|\sigma_k[V_{\theta^{(i, 0)}}] - \sigma^{(i,0)}(k)\bigg| = \bigg|\sigma_k[V_{\theta^{(i, 0)}}] - \Big(z^{(i,0)}(k) - w^{(i,0)}(k)^2\Big)\bigg| \le 
	c_1' (1-\gamma)^{-2}\cdot \sqrt{\frac{\log[\delta^{-1}]}{m}} 
	\]
	for some constant $c_1'$, where $\sigma^{(i,0)}(k):=z^{(i,0)} -(w^{(i,0)}(k))^2$ according to tep 13 of Alg.\ 2.
	Thus,
	$\sigma_k[V_{\theta^{( i, 0)}}]\le  \sigma^{(i,0)}(k) 
	+ c_1' (1-\gamma)^{-2}\cdot \sqrt{\frac{\log[\delta^{-1}]}{m}}$. 
	We further obtain,
	\begin{align*}
	&\sqrt{\sigma^{(i,0)}(k) 
		+ c_1' (1-\gamma)^{-2}\cdot \sqrt{\frac{\log[\delta^{-1}]}{m}}}\le \sqrt{\sigma^{(i,0)}(k)} + \Big(c_1'^2 (1-\gamma)^{-4}\frac{\log[\delta^{-1}]}{m}\Big)^{1/4}.
	\end{align*}
	By plugging in $\delta\gets \delta/(KR'R)$, we have, 
	\begin{align*}
	\bigg|w^{(i,0)}(k) - P(\cdot|s_k,a_k)^\top V_{\theta^{(i,0)}}\bigg|&\le 
	 c_1 \sqrt{\frac{\log[KRR'\delta^{-1}]\sigma_k[V_{\theta^{(i, 0)}}]}{m}}  + \frac{c_1\log(KRR'\delta^{-1})}{(1-\gamma)m}\\
	 &\le  \Theta\bigg[\sqrt{\frac{\log[R'RK\delta^{-1}]\cdot\sigma^{(i,0)}(k)}{ m}} + \frac{\log[R'RK\delta^{-1}]}{(1-\gamma)m^{3/4}}\bigg]\\
	 &= \epsilon^{(i,0)}(k)
	\end{align*}
	with probability at least $1-\delta/(KR'R)$,
	where $\epsilon^{(i,0)}(k)$ is defined in Step 13 of Algorithm~\ref{alg:learn}.
	Since
	$\sigma^{(i,0)}(k) \le \sigma_k[V_{\theta^{(i, 0)}}]  
	+ c_1' (1-\gamma)^{-2}\cdot \sqrt{\frac{\log[\delta^{-1}]}{m}}$, we further have
	\begin{align*}
	\epsilon^{(i,0)}(k)
	&\le \Theta\bigg[\sqrt{\log(RR'K\delta^{-1})\sigma_k[V_{\theta^{(i,0)}}]/m} + \bigg( (1-\gamma)^{-4}\frac{\log[RR'K\delta^{-1}]^4}{m^3}\bigg)^{1/4}\bigg].
	\end{align*}
	Therefore, by applying an union bound over all $k\in [K]$, we have
	\[
	\PP[\cE^{(i,0)}|\cG^{(i)}] \ge 1-\delta/(RR').
	\]
	Reminder that if $\cE^{(i,0)}$ happens, then
	$w^{(i,0)} - \epsilon^{(i,0)} \le P(\cdot|s_k,a_k)^\top V_{\theta^{(i,0)}}$.
	
	\textbf{\\Probability of $\cE^{(i,j)}$ by Induction}:	
	We now prove by induction that
	\begin{align}
	\label{eqn:induction_events}
	\PP[\cE^{(i,j)}| \cE^{(i,j-1)}, \cE^{(i,j-2)},\ldots, \cE^{(i,0)}, \cF^{(i)}] \ge 1-\delta/(RR').
	\end{align}
	For the base case $j=1$, we have
	\begin{align*}
	w^{(i,1)} &= w^{(i,0)}\quad \text{and} \quad
	\epsilon^{(i,1)} = \epsilon^{(i,0)} + \Theta(1-\gamma)^{-1}2^{-i}\sqrt{\log(RR'K/\delta)},
	\end{align*}
	therefore $\PP[\cE^{(i,1)}| \cE^{(i,0)}, \cG^{(i)}] = 1$.
	Now consider $j$. Conditioning on $\cE^{(i,j-1)}, \cE^{(i,j-2)},\ldots, \cE^{(i,0)}, \cF^{(i)}$,  we have
	with probability at least $1-\delta$,
	\begin{align*}
	\Big|\frac{1}{m_1} &\sum_{\ell=1}^{m_1}
	\Big(
	V_{\theta^{(i, j-1)}}(x_k^{(\ell)}) -
	V_{\theta^{(i,0)}}(x_k^{(\ell)})
	\Big) - P(\cdot|s_k,a_k)^\top\Big(
	V_{\theta^{(i, j-1)}} -
	V_{\theta^{(i, 0)}}
	\Big)
	\Big| \\
	&\le c_2\max_{s}|V_{\theta^{( i, j-1)}}(s) -
	V_{\theta^{( i, 0)}}(s)|\cdot \sqrt{\frac{\log(\delta^{-1})}{m_1}}
	\\
	&\le c_2\max_{s}|v^*(s) -
	V_{\theta^{(i, 0)}}(s)|\cdot \sqrt{\log(\delta^{-1})/m_1}
	\qquad\qquad\text{\codecmt{$V_{\theta^{(i,0)}}\le V_{\theta^{(i,j-1)}}\le v^*$}}
	\\
	&\le c_22^{-i}(1-\gamma)^{-1}\cdot \sqrt{\log(\delta^{-1})/m_1}.  \qquad\qquad\qquad\qquad\text{\codecmt{By definition of $\cG^{(i)}$}}
	\end{align*}
	Letting $\delta\gets \delta/(RR'K)$ and applying a union bound over $k\in[K]$, we obtain \eqref{eqn:induction_events}.\\
	
\noindent\textbf{Probability of Joint Events}:	
	Finally, we have that 
	{\small
	\begin{align*}
	\Pr[\cE^{(i,0)}\cap \cE^{(i,1)} \ldots \cap \cE^{(i,R)}|\cG^{(i)}] &= 
	\Pr[\cE^{(i,0)}|\cG^{(i)}]\Pr[\cE^{(i,1)}|\cE^{(i,0)}, \cG^{(i)}]\ldots\Pr[\cE^{(i,R)}| \cE^{(i,0)}, \cE^{(i,1)}, \ldots, \cE^{(i,R-1)}, \cG^{(i)}] \\
	&\ge 1-\delta / R'.
	\end{align*}}
\end{proof}

\begin{lemma}[Upper Bound of $\epsilon^{(i,j)}(k)$]\label{lemma-10}
	Conditioning on the events $\cF^{(i)}$, $\cE^{(i,0)}, \cE^{(i,1)}, \ldots, \cE^{(i,j)}$,
	we have, for all $k\in [K]$
	\begin{align*}
	\epsilon^{(i,j)}(k)&\le C\bigg[\sqrt{\frac{\log(R'RK\delta^{-1})\sigma_k[v^*]}{m}} + \frac{\log(R'RK\delta^{-1})}{(1-\gamma)m^{3/4}}
	+2^{-i}\sqrt{\frac{\log(R'RK\delta^{-1})}{(1-\gamma)^2m_1}}
	\bigg]
	\end{align*}
	for some universal constant $C>0$.
\end{lemma}

\begin{proof}
	Conditioning on  $\cF^{(i)}$, $\cE^{(i,0)}, \cE^{(i,1)}, \ldots, \cE^{(i,j)}$, we have
	\begin{align*}
	\epsilon^{(i,0)}(k) &\le
	c_1\Bigg[\sqrt{\frac{\log(R'RK\delta^{-1})\sigma_k[V_{\theta^{( i, 0)}}]}{m}} + \frac{\log(R'RK\delta^{-1})}{(1-\gamma)m^{3/4}}\Bigg]\\
	&\le 
	c_1'\Bigg[\sqrt{\frac{\log(R'RK\delta^{-1})\sigma_k[v^*]}{m}} + \frac{\log(R'RK\delta^{-1})}{(1-\gamma)m^{3/4}}+2^{-i}\sqrt{\frac{\log(R'RK\delta^{-1})}{(1-\gamma)^2m}}\Bigg],
	\end{align*}
	for some generic constants $c_1, c_1'$, where we use the fact that
	$\|V_{\theta^{(i,0)}} - v^*\|_{\infty}\le 2^{-i}/(1-\gamma)$ and the triangle inequality.
	Using the definition of $\epsilon^{(i,j)}$ and the fact $m_1\leq m$, we have
	\begin{align*}
	\epsilon^{(i,j)}(k) &= \epsilon^{(i,0)}(k)
	+ c_2 2^{-i}\sqrt{\frac{\log(R'RK\delta^{-1})}{(1-\gamma)^2m_1}}\\
	&\le 
	c_2'\Bigg[\sqrt{\frac{\log(R'RK\delta^{-1})\sigma_k[v^*]}{m}} + \frac{\log(R'RK\delta^{-1})}{(1-\gamma)m^{3/4}}+2^{-i}\sqrt{\frac{\log(R'RK\delta^{-1})}{(1-\gamma)^2m_1}}\Bigg],
	\end{align*}
	for some generic constants $c_2, c_2'$, where we use the fact that $m\ge m_1$. 
	This concludes the proof.
\end{proof}

\subsection{Error Accumulation in One Outer Iteration}

\begin{lemma}
	\label{lemma:f-event}
	For $i=0, 1, 2, \ldots, R'$,
	$
	\PP[\cG^{(i+1)}|\cG^{(i)}] \ge 1 -\delta/(R'+1).
	$
\end{lemma}
\begin{proof}[Proof of Lemma~\ref{lemma:f-event}]
	Conditioning on $\cG^{(i)}$, suppose that the events $\cE^{(i,0)}, \cE^{(i,1)}, \ldots, \cE^{(i,R)}$ all happen, which has probability at least $1-\delta/R'$ according to Lemma \ref{lemma-9}.
	For any $s\in \cS$, we analyze the total error accumulated in the $i$-th outer iteration:
	\begin{align*}
	v^{*}(s)- V_{\theta^{(i, j)}}(s)
	&\le \gamma P^{\pi^*}(\cdot|s)^\top(v^* -V_{\theta^{(i, j-1)}})+ 2\gamma \sum_{k}\phi_k(s,\pi^*(s))\epsilon^{(i,j)}(k)
	\qquad\qquad\text{\codecmt{Lemma~\ref{lemma:mono}}}\\
	&\le \gamma^2 \sum_{s'}P^{\pi^*}(s'|s)^\top P^{\pi^*}(\cdot|s')^\top(v^* -V_{\theta^{(i, j-2)}})+ 
	2\gamma^2 P^{\pi^*}(\cdot|s)^\top 
	\sum_{k}\phi_k(\cdot,\pi^*(\cdot))\epsilon^{(i,j-1)}(k)
	\\
	&\qquad+
	2\gamma \sum_{k}\phi_k(s,\pi^*(s))\epsilon^{(i,j)}(k)
	\hspace{2.25cm}\text{\codecmt{applying Lemma~\ref{lemma:mono} again on $v^* - V_{\theta^{(i,j-1)}}$}}
	\\
	&\le\ldots\qquad\qquad\qquad\qquad\qquad\qquad\qquad\qquad~\qquad\text{\codecmt{applying Lemma~\ref{lemma:mono} recursively}}
	\\
	\end{align*}
	\begin{align*}
	\hspace{2.5cm}&\le 
	\gamma^{j}[\big(P^{\pi^*}\big)^{j}(v^* -V_{\theta^{(i, 0)}})](s)+ 2\sum_{j'=0}^{j-1} 
	\gamma^{j'+1} \sum_{k,s'}\big(P^{\pi^*}\big)^{j'}_{s,s'}\phi_k(s',\pi^*(s'))\epsilon^{(i,j-j')}(k)\\
	&\le 
	\gamma^{j}(1-\gamma)^{-1}
	+ C\sum_{j'=0}^{j-1} \gamma^{j'+1}\Bigg[
	\frac{\log(R'RK\delta^{-1})}{(1-\gamma)m^{3/4}}+2^{-i}\sqrt{\frac{\log(R'RK\delta^{-1})}{(1-\gamma)^2m_1}}\Bigg]\\
	&\qquad+ C\sum_{j'=0}^{j-1} 
	\gamma^{j'+1} \sum_{s'}\big(P^{\pi^*}\big)^{j'}_{s,s'}\cdot\sum_k\phi_k(s',\pi^*(s'))\sqrt{\frac{\log(R'RK\delta^{-1})\sigma_{k}[v^*]}{m}}
	\\ 
	&\qquad\qquad\quad\text{\codecmt{using $\|v^* -V_{\theta^{(i, 0)}}\|_{\infty}\leq\frac1{1-\gamma}$ and the upperbound of $\epsilon^{(i, j)}$ (Lemma \ref{lemma-10})}}\\
	&\le 
	\gamma^{j}(1-\gamma)^{-1}
	+ C\sum_{j'=0}^{j-1} \gamma^{j'+1}\Bigg[
	\frac{\log(R'RK\delta^{-1})}{(1-\gamma)m^{3/4}}+2^{-i}\sqrt{\frac{\log(R'RK\delta^{-1})}{(1-\gamma)^2m_1}}\Bigg]\\
	&\qquad+ C\sum_{j'=0}^{j-1} 
	\gamma^{j'+1} \sum_{s'}\big(P^{\pi^*}\big)^{j'}_{s,s'}\cdot\sqrt{\frac{\log(R'RK\delta^{-1})\sigma_{s', \pi^*(s')}[v^*]}{m}}\\
	& \qquad\qquad\qquad\qquad\qquad\qquad\qquad\qquad\qquad\qquad\quad\text{\codecmt{applying Lemma~\ref{lemma:var}}}\\
	&= 
	\gamma^{j}(1-\gamma)^{-1}
	+ C\frac{1-\gamma^j}{1-\gamma}\cdot\Bigg[
	\frac{\log(R'RK\delta^{-1})}{(1-\gamma)m^{3/4}}
	+2^{-i}\sqrt{\frac{\log(R'RK\delta^{-1})}{(1-\gamma)^2m_1}}\Bigg]+ \\
	&\qquad C\sum_{j'=0}^{j-1} 
	\gamma^{j'+1} \sum_{s'}\big(P^{\pi^*}\big)^{j'}_{s,s'}\cdot\sqrt{\frac{\log(R'RK\delta^{-1})\sigma_{s', \pi^*(s')}[v^*]}{m}},
	\end{align*}
	where $C$ is a generic constant.
	By Lemma~C.1 of \cite{sidford2018near} (a form of law of total variance for the Markov chain under $\pi^*$), we have,
	\[
	\sum_{j'=0}^{j-1} 
	\gamma^{j'+1} \sum_{s'}\big(P^{\pi^*}\big)^{j'}_{s,s'}\sqrt{\sigma_{s', \pi^*(s')}[v^*]}
	\le C'\sqrt{(1-\gamma)^{-3}}
	\]
	for some generic constant $C'$.
	Combining the above equations,  
	and setting 
	\begin{align*}
	m &= C''\frac{1}{\epsilon^2}\cdot\frac{\log (R'RK\delta^{-1})^{4/3}}{(1-\gamma)^{3}}
	\quad\text{and}\quad
	m_1 = C''\cdot \frac{\log (R'RK\delta^{-1})}{(1-\gamma)^{2}},
	\end{align*}
	$
	R \ge \Theta[i\cdot(1-\gamma)^{-1}]
	$ and $2^{-i}/(1-\gamma)\ge \Theta(\epsilon)$ for some generic constant $C''$, we can make the accumulated error as small as
	\[
	v^{*}(s)- V_{\theta^{(i, R)}}(s)
	\le c2^{-i}/(1-\gamma)
	\]
	for some $c>0$.
	Since $V_{\theta^{(i+1,0)}}(s) = V_{\theta^{(i,R)}}(s)$ together with the monotonicity properties shown in Lemma~\ref{lemma:mono}, we obtain that  conditioning on $\cG^{(i)}, \cE^{(i,0)}, \cE^{(i,1)}, \ldots, \cE^{(i,R)}$, the event $\cG^{(i+1)}$ happens with probability 1.
\end{proof}

\subsection{Proof of Theorem~\ref{thm:main-thm}}

\begin{proof}[Proof of Theorem~\ref{thm:main-thm}]
	Conditioning on $\cG^{(R')}$, 
	we have 
	\[
	\forall s\in \cS: \quad 
	0\le v^*(s) - V_{\theta^{( R', R)}}(s)
	\le 2^{-R'}/(1-\gamma). 
	\]
	Since $R' = \Theta(\log[\epsilon^{-1}(1-\gamma)^{-1}])$, we have
	$|v^*(s) - V_{\theta^{( R', R)}}(s)|$ $\le \epsilon$.
	Moreover, we have	 
	$$ v^*(s) -\epsilon \le V_{\theta^{(R', R)}}(s) \le \cT_{\pi_{\theta^{(R', R)}}}
	V_{\theta^{(R', R)}}[s] \le v^{\pi_{\theta^{(R', R)}}}[s]\le v^*(s),$$
	where the third inequality follows from monotonicity of $\cT_{\pi^{(R', R)}}$. 
	Therefore
	$\pi_{\theta^{( R', R)}}$ is an $\epsilon$-optimal policy from any initial state $s$.
	Notice that $\PP[\cG^{(i)}|\cG^{(i-1)}]\ge 1-\delta/R'$, we have $\PP[\cG^{(R')}]\ge \PP[\cG^{(R')}\cap \cG^{(R-1)}\cap\ldots \cG^{(0)}] \ge 1-\delta$.
	Finally, one can show the main result by counting the number of samples needed by the algorithm.  
\end{proof}

\end{document}